\documentclass[twoside]{article}

\usepackage[accepted]{aistats2024}

\usepackage[round]{natbib}

\RequirePackage{xcolor} %
\RequirePackage{algorithm}
\RequirePackage{algorithmic}
\RequirePackage{eso-pic} %
\RequirePackage{forloop}
\RequirePackage{url}

\usepackage{microtype}
\usepackage{graphicx}
\usepackage{subcaption}
\usepackage{multirow}
\usepackage{booktabs}
\usepackage{tabularx}
\usepackage{color}

\usepackage{hyperref}

\usepackage[normalem]{ulem}

\usepackage{amsmath}
\usepackage{amssymb}
\usepackage{mathtools}
\usepackage{amsthm}

\usepackage{bm}

\usepackage{booktabs}
\usepackage{longtable}
\usepackage{pdflscape}

  \theoremstyle{plain}
  \newtheorem{Theorem}{\protect\theoremname}
  \theoremstyle{plain}
  \newtheorem*{Theorem*}{\protect\theoremname}
  \theoremstyle{plain}
  \newtheorem{proposition}{\protect\propositionname}
  \theoremstyle{plain}
  \newtheorem*{prop*}{\protect\propositionname}
  \theoremstyle{plain}
  \newtheorem{lemma}{\protect\lemmaname}
   \theoremstyle{plain}
  \newtheorem*{lemma*}{\protect\lemmaname}  
  \theoremstyle{plain}
  \newtheorem{definition}{\protect\definitionname}
  \theoremstyle{plain}
  \newtheorem{corollary}{\protect\corollaryname}
  \theoremstyle{plain}
  
 \theoremstyle{plain}
\newtheorem{remark}{Remark}
\theoremstyle{plain}
\newtheorem{fact}{\protect\factname}
\theoremstyle{plain}
\newtheorem{assumption}{\protect\assumptionname}
\theoremstyle{plain}

 \theoremstyle{plain}
\newtheorem*{model*}{Model}
\theoremstyle{plain}
\newtheorem{conjecture}{\protect\conjecturename}

\DeclareMathOperator{\Tr}{Tr}
\DeclareMathOperator*{\argmax}{arg\,max}  %
\newcommand{\states}{\mathcal{S}}
\newcommand{\trans}{P}
\newcommand{\actions}{\mathcal{A}}
\newcommand{\norm}[1]{\left\|#1\right\|}
\def\mI{{\bm{I}}}
\newcommand{\abs}[1]{\left|#1\right|}

\def\R{\mathbb{R}}
\newcommand{\N}{\mathbb{N}}
\newcommand{\Fc}{\mathcal{F}}
\newcommand{\Xc}{\mathcal{X}}

\newcommand{\Pc}{\mathcal{P}}
\def\rmGamma{{\boldsymbol{\Gamma}}}
\def\mSigma{{\bm{\varSigma}}}
\newcommand{\poly}{\mathrm{poly}}

\newcommand{\given}[1][]{\:#1\vert\:}
\newcommand{\E}[2][]{\mathbb{E}_{#1}\left[#2\right]} %
\newcommand{\prob}{\mathbb{P}}
\newcommand{\var}[2][]{{ \mathbb{V}_{#1}\left(#2\right)}}

\newcommand\independent{\protect\mathpalette{\protect\independenT}{\perp}}
\def\independenT#1#2{\mathrel{\rlap{$#1#2$}\mkern2mu{#1#2}}}

\newcommand{\ohist}[1]{{ \mathcal{H}^o_{#1} }}
\newcommand{\uhist}[1]{{ \mathcal{H}^u_{#1} }}

\newcommand{\hist}[1]{{ \mathcal{H}_{#1} }}

\newcommand{\V}[2]{{V^{#1}_{#2}}}
\newcommand{\barV}[2][]{{\widebar{V}^{#1}_{#2}}}

\def\<#1,#2>{\left\langle #1,#2 \right\rangle}

\newcounter{protocol}


\makeatletter
\DeclareRobustCommand\widecheck[1]{{\mathpalette\@widecheck{#1}}}
\def\@widecheck#1#2{%
    \setbox\z@\hbox{\m@th$#1#2$}%
    \setbox\tw@\hbox{\m@th$#1%
       \widehat{%
          \vrule\@width\z@\@height\ht\z@
          \vrule\@height\z@\@width\wd\z@}$}%
    \dp\tw@-\ht\z@
    \@tempdima\ht\z@ \advance\@tempdima2\ht\tw@ \divide\@tempdima\thr@@
    \setbox\tw@\hbox{%
       \raise\@tempdima\hbox{\scalebox{1}[-1]{\lower\@tempdima\box
\tw@}}}%
    {\ooalign{\box\tw@ \cr \box\z@}}}
\makeatother

\makeatletter
\let\save@mathaccent\mathaccent
\newcommand*\if@single[3]{%
  \setbox0\hbox{${\mathaccent"0362{#1}}^H$}%
  \setbox2\hbox{${\mathaccent"0362{\kern0pt#1}}^H$}%
  \ifdim\ht0=\ht2 #3\else #2\fi
  }
\newcommand*\rel@kern[1]{\kern#1\dimexpr\macc@kerna}
\newcommand*\widebar[1]{\@ifnextchar^{{\wide@bar{#1}{0}}}{\wide@bar{#1}{1}}}
\newcommand*\wide@bar[2]{\if@single{#1}{\wide@bar@{#1}{#2}{1}}{\wide@bar@{#1}{#2}{2}}}
\newcommand*\wide@bar@[3]{%
  \begingroup
  \def\mathaccent##1##2{%
    \let\mathaccent\save@mathaccent
    \if#32 \let\macc@nucleus\first@char \fi
    \setbox\z@\hbox{$\macc@style{\macc@nucleus}_{}$}%
    \setbox\tw@\hbox{$\macc@style{\macc@nucleus}{}_{}$}%
    \dimen@\wd\tw@
    \advance\dimen@-\wd\z@
    \divide\dimen@ 3
    \@tempdima\wd\tw@
    \advance\@tempdima-\scriptspace
    \divide\@tempdima 10
    \advance\dimen@-\@tempdima
    \ifdim\dimen@>\z@ \dimen@0pt\fi
    \rel@kern{0.6}\kern-\dimen@
    \if#31
      \overline{\rel@kern{-0.6}\kern\dimen@\macc@nucleus\rel@kern{0.4}\kern\dimen@}%
      \advance\dimen@0.4\dimexpr\macc@kerna
      \let\final@kern#2%
      \ifdim\dimen@<\z@ \let\final@kern1\fi
      \if\final@kern1 \kern-\dimen@\fi
    \else
      \overline{\rel@kern{-0.6}\kern\dimen@#1}%
    \fi
  }%
  \macc@depth\@ne
  \let\math@bgroup\@empty \let\math@egroup\macc@set@skewchar
  \mathsurround\z@ \frozen@everymath{\mathgroup\macc@group\relax}%
  \macc@set@skewchar\relax
  \let\mathaccentV\macc@nested@a
  \if#31
    \macc@nested@a\relax111{#1}%
  \else
    \def\gobble@till@marker##1\endmarker{}%
    \futurelet\first@char\gobble@till@marker#1\endmarker
    \ifcat\noexpand\first@char A\else
      \def\first@char{}%
    \fi
    \macc@nested@a\relax111{\first@char}%
  \fi
  \endgroup
}
\makeatother
\renewcommand{\leq}{\leqslant}
\renewcommand{\le}{\leqslant}
\renewcommand{\geq}{\geqslant}
\renewcommand{\ge}{\geqslant}

\usepackage{amsthm}
\usepackage{amsmath}
\usepackage{amssymb}
\usepackage{thm-restate}

\usepackage{babel}
\providecommand{\assumptionname}{Assumption}
\providecommand{\definitionname}{Definition}
\providecommand{\lemmaname}{Lemma}
\providecommand{\propositionname}{Proposition}
\providecommand{\corollaryname}{Corollary}
\providecommand{\examplename}{Example}
\providecommand{\factname}{Fact}
\providecommand{\conditionname}{Condition}
\providecommand{\theoremname}{Theorem}
\providecommand{\conjecturename}{Conjecture}

\usepackage[capitalize,noabbrev]{cleveref}
\crefname{Theorem}{Theorem}{Theorems}
\crefname{conjecture}{Conjecture}{Conjectures}
\crefname{fact}{Fact}{Facts}
\crefname{assumption}{Assumption}{Assumptions}

\hypersetup{
    colorlinks = true,
    linkcolor = {blue},
    citecolor = {blue}
}

\usepackage[breakable]{tcolorbox}
\definecolor{block-gray}{gray}{0.9}
\newtcolorbox{blockquote}{colback=orange!15!white,grow to right by=1.5mm,grow to left by=1.5mm,boxrule=0pt,boxsep=0pt}

\begin{document}

\runningtitle{Prior-dependent analysis of PSRL}

\twocolumn[

  \aistatstitle{Prior-dependent analysis of posterior sampling reinforcement learning with function approximation}

  \aistatsauthor{Yingru Li \and Zhi-Quan Luo}

  \aistatsaddress{
    \href{mailto:yingruli@link.cuhk.edu.cn}{yingruli@link.cuhk.edu.cn} \\
    The Chinese University of Hong Kong, Shenzhen, China\\ Shenzhen Research Institute of Big Data } ]

\begin{abstract}
  This work advances randomized exploration in reinforcement learning (RL) with function approximation modeled by linear mixture MDPs. We establish the first prior-dependent Bayesian regret bound for RL with function approximation;
  and refine the Bayesian regret analysis for posterior sampling reinforcement learning (PSRL), presenting an upper bound of ${\mathcal{O}}(d\sqrt{H^3 T \log T})$, where $d$ represents the dimensionality of the transition kernel, $H$ the planning horizon, and $T$ the total number of interactions. This signifies a methodological enhancement by optimizing the $\mathcal{O}(\sqrt{\log T})$ factor over the previous benchmark \citep{osband2014model} specified to linear mixture MDPs. Our approach, leveraging a value-targeted model learning perspective, introduces a decoupling argument and a variance reduction technique, moving beyond traditional analyses reliant on confidence sets and concentration inequalities to formalize Bayesian regret bounds more effectively.
\end{abstract}

\section{Introduction}
\label{sec:intro}

Reinforcement learning (RL) has become a cornerstone in the development of intelligent systems~\citep{sutton2018reinforcement}, propelling the forefront of artificial intelligence and decision sciences towards creating agents capable of making autonomous decisions in dynamically complex environments. The crux of advancing RL has been the strategic integration of function approximation techniques and the assimilation of prior knowledge. Function approximation methods have transcended traditional RL boundaries~\citep{bertsekas1996neuro}, enabling scalable solutions across extensive or continuous state spaces, a necessity for practical real-world problem-solving. Concurrently, the infusion of prior knowledge—ranging from domain expertise to insights gleaned from historical data~\citep{wang2018exponentially,wang2019divergence,agarwal2020optimistic} or pre-trained models~\citep{yang2023foundation}—into RL algorithms has catalyzed a transformative leap in learning efficiency. This leap is underscored by providing RL systems with a foundational understanding of advantageous actions, even before any environment interaction occurs.

\paragraph{Motivation.} The nuanced application of priors within the RL paradigm, especially through Bayesian methodologies, ushers in a sophisticated equilibrium between exploration and exploitation—essential for the practical deployment of RL. By preemptively incorporating knowledge about the environment's dynamics through prior distributions, RL algorithms are primed for more informed exploration. Despite its immense potential, the exploration of priors' role, particularly within the context of function approximation, remains scant. This uncharted domain presents a ripe opportunity for enhancing RL's learning outcomes and applicability in diverse settings, thus pushing the envelope of RL's capabilities and efficiency.

\paragraph{Important Question.} Against this backdrop, we are compelled to address a pivotal inquiry:
\begin{quote}
    \emph{How can the symbiosis of prior knowledge and function approximation be optimized to elevate the adaptability and efficiency of RL algorithms?}
\end{quote}

\subsection{Key Contributions}
Our exploration into this inquiry yields several seminal contributions, especially within the context of linear mixture Markov Decision Processes:
\begin{itemize}
    \item The introduction of a prior-dependent Bayesian regret bound (\cref{thm:inhomo-bayesian-regret-0}) within the realm of RL with function approximation, elucidating the impact of prior distribution variance on learning efficiency. This milestone underscores a deeper comprehension of RL dynamics in scenarios where environmental distributions are known or can be approximated.
    \item The unveiling of a prior-free Bayesian regret bound for Posterior Sampling for Reinforcement Learning (PSRL) in linear mixture MDPs, i.e., ${\mathcal{O}}(d H^{3/2} \sqrt{T \log T})$ Bayesian regret bound for PSRL in linear mixture MDPs (\cref{rem:prior-free-0}), where $d$ is the dimension of feature in basis transition kernels, $H$ is the planning horizon and $T$ is the number of interactions. Our upper bound improves the Bayesian regret bound of PSRL by traditional by a $O({\log T})$ factor.
    \item The advancement of the analytical landscape of Bayesian regret in RL through methodological novelties, including a decoupling lemma and a variance reduction theorem. These innovations engender a nuanced perspective of regret dynamics, advancing beyond the conventional frameworks reliant on confidence sets and concentration inequalities.
\end{itemize}

\subsection{Preview of Technical Novelty}
Our theoretical contributions are underpinned by substantial technical innovations:

\emph{Posterior variance reduction.} Our analysis features a variance reduction theorem that addresses the heterogeneity of value variance in RL, highlighting a predictable reduction in posterior variance under specific conditions.
from the value-targeted perspective of model learning in \cref{sec:value_pers}.
In expectation, the posterior variance of the true model parameters will be reduced in a non-uniform way due to the heteroscedasticity nature of the value variance.
The posterior variance is predictable when we have access to the true posterior distribution under correct prior and correct likelihood distribution.
Our posterior variance reduction argument provides a way to measure how epistemic uncertainty is reduced after new observations are informed in the learning progresses.

\emph{Decoupling argument.} We introduce a decoupling argument (\cref{lem:linear-decoupling-0,lem:linear-decoupling}) that enhances our understanding of Bayesian regret by separating the intertwined dynamics of action selection and value estimation and relate the Bayesian regret to posterior variance. The idea of decoupling is inspired by information-theoretic analysis \citep{russo2016information, kalkanli2020improved} for the linear bandit.
The statement in \cref{lem:linear-decoupling} is slightly stronger than previous literature, which is essential to prove a better $H$ dependence.
In details, this stronger statement is required to upper bound the \emph{absolute value-targeted error} in \cref{prop:regret-decompose} and consequently upper bound the \emph{absolute estimation error} in \cref{prop:sum-virtual-value-variance}.
Extending the previous statement in bandit literature to this stronger statement is not complicated but new.

\emph{Integration of prior knowledge:} We provide a unique characterization of the relationship between the regret bound and the prior distribution, offering new insights into integrating prior knowledge into RL algorithms.
In contrast to previous analysis that utilizes confidence sets, we leverage the decoupling argument and variance reduction argument to upper bound the Bayesian regret, essential for deriving the prior-dependent analysis.
Previous Bayesian analysis in RL \citep{osband2014model,osband2017posterior} relies on a translation of frequentist concentration bound to a Bayesian credit interval, which is essentially a prior-free argument.

Besides, as a by-product of our analysis for the prior-free result, because of going beyond traditional analysis relying on the confidence-sets or the notion of eluder dimension~\citep{osband2014model,ayoub2020model} to upper bound regret, the final regret bound will be naturally improved by a $O({\log T})$ factor\footnote{We also discuss a conjecture on the $\log T$ term in the lower bound~\cref{sec:lower-bound-conjecture}.}.

\subsection{Related Works}
\label{sec:related_work}

Our endeavor distinctly positions itself within the vast expanse of RL research by concentrating on the synergistic integration of prior knowledge and function approximation in linear mixture MDPs. Contrasting with preceding works that predominantly focus on tabular contexts or marginally engage with the potential of prior knowledge, our comprehensive Bayesian analysis introduces both prior-dependent and prior-free regret bounds. This dual-faceted approach meticulously fills a significant lacuna in existing literature, shedding light on the efficacious incorporation of prior knowledge into RL strategies.

\subparagraph{Linear function approximation.}
Linear mixture MDPs~\citep{ayoub2020model} become a widely-accepted benchmark to understand the synergy of exploration and model-based function approximation. Many algorithms including value-targeted regression~\citep{ayoub2020model} and policy optimization~\citep{cai2020provably} are developed for understand the statistical complexity of reinforcement learning with linear function approximation under this setup.
Existing algorithm for linear mixture MDPs that achieves the near-minimax optimal dependency\footnote{We focus on the time-inhomogeneous MDPs in this work. The minimax lower bound for time-inhomogeneous linear mixture MDPs~\citep{zhou2021nearly} is $\Omega(H d\sqrt{T})$. Here $d$ is the dimension of features for basis transition kernels.}
is based on optimism in the face of uncertainty (OFU) principle~\citep{zhou2021nearly}. For another setup of linear MDPs ~\citep{yang2019sample,jin2020provably}, a line of works are proposed to understand value-based function approximation and exploration~\citep{jin2020provably,agarwal2023vo}.

\subparagraph{Randomized exploration.}
An alternative category of algorithms \citep{strens2000bayesian}, inspired by Thompson sampling \citep{thompson1933likelihood}, involves
randomized exploration that samples a set of statistically plausible action values and selects the maximizing action.
Practically, randomized exploration shows promising computational and statistical advantages \citep{chapelle2011empirical,osband2019deep,li2022hyperdqn,li2024hyperagent}.
Empirical success has prompted a surge of interest in theoretical analysis for randomized exploration, e.g. randomized least-square value iteration (RLSVI)~\citep{zanette2020frequentist,osband2019deep} and its optimistic sampling~\citep{ishfaq2021randomized} or approximate sampling variant~\citep{ishfaq2024provable,li2024hyperagent}.
These analyses in linear MDPs mostly rely on Azuma–Hoeffding concentration inequality that is unaware of value variance.
For linear mixture MDPs, no specific frequentist analyses of randomized exploration exist.
Notebly, the HyperAgent~\citep{li2024hyperagent} demonstrates state-of-the-art practical efficiency as well as provable sublinear regret and scalable per-step computation, bridging theory and practice.

\subparagraph{Bayesian regret analysis.}
The Bayesian regret analysis in \citep{osband2017posterior,lu2019information,li2024hyperagent} requires independent Dirichlet prior imposed on the transition probabilities, tailored for tabular MDPs.
\citet{lu2021reinforcement} requires the independent Beta prior on the transition probabilities since they only analyze a special ring-structured MDPs provided a conjecture holds.
Existing Bayesian analysis for RL with function approximation did not give detailed characterization on relationship between regret bound and the prior distribution. \citet{osband2014model} provide the prior-free Bayesian regret for RL with function approximation, borrowing the notion of eluder dimension~\citep{russo2014learning}. Specifically, \citet{osband2014model} gives a Bayesian regret bound $\widetilde{\mathcal{O}}( \E{K^*} \sqrt{d_{K} d_{E} H T \log T })$ for PSRL with general function approximation, where $K^*$, $d_{K}$ and $d_{E}$ are the global Lipschitz constant for the future value function, the Kolmogorov and the Eluder dimensions of the model class.
In the case of linear mixture MDPs, this Bayesian regret becomes $\tilde{\mathcal{O}}(d \sqrt{H^3T} \log{T})$
\footnote{By the property of eluder dimension and Kolmogorov dimension, $\E{K^*} = \mathcal{O}(H), d_K = \mathcal{O}{(d)}, d_E = \dim_E( \mathcal{F}, T^{-1} ) = \mathcal{O}(d\log T)$ in the class of linear mixture MDPs where $d$ is the dimension of feature in basis kernels and $T$ is the total interaction rounds.
}.
In contrast to the existing analyses, our analysis provides the first prior-dependent Bayesian analysis of PSRL and an improved prior-free bound under function approximation; see \cref{thm:inhomo-bayesian-regret-0} and \cref{rem:prior-free-0}.

\section{Preliminaries}
\label{sec:problem_formulation}
We consider the problem
where the RL agent learns to optimize the \emph{finite horizon MDP} over repeated episodes of interactions.
There are two sources of randomness in this repeated process: the \emph{environmental randomness} and the \emph{algorithmic randomness}.

\paragraph{Finite horizon MDP.} A finite horizon Markov Decision Process (MDP)
\footnote{We do not assume finite state-action spaces. The state space $\states$ and action space $\actions$ can be unbounded.} \citep{puterman2014markov} is described by $M = \left( \states, \actions, \trans, R, H, \rho \right)$, where $\states$ is the state space,  $\actions$ is the action space,  $H$ is the horizon.
In the beginning of an episode, an initial state $s_0 \in \states$ is sampled from $\rho$. At stage $0 \le h \le H-1$, given state $s_h$ and action $a_h$, the agent observes a reward $r_{h+1}$ sampled from the reward distribution with mean $R(h, s_h, a_h) \in [0, 1]$ and a next state $s_{h+1}$ sampled from transition kernel $\trans(h, s_{h}, a_{h})$.

A deterministic policy $\pi = \{\pi_h\}_{h=0}^{H-1} $ is a collection of $H$ functions, each of which maps from a state  $s_h \in \states$ to an action $a_h \in \actions$.
For an MDP $\hat{M}$ and a policy $\pi$, we define the $h$-stage action-value function $Q_{\pi, h}^{\hat{M}} (s, a)$ and the $h$-stage value function $V_{\pi, h}^{M}$ for every $(s, a) \in \states \times \actions$:
\begin{align*}
  Q_{\pi, h}^{\hat{M}} (s, a) & : = \E{ \sum_{j = h}^{H-1} r_{j+1} \given[\bigg] M = \hat{M}, s_h = s, a_h = a, \pi }, \\
  V_{\pi, h}^{\hat{M}}(s)     & := Q_{\pi, h}^{\hat{M}} (s, \pi_h(s)), \quad \forall h =0, 1, \cdots H-1.
\end{align*}
The conditional expectation specifies that actions from stage $h+1$ to $H-1$ are determined according to the policy $\pi$ when interacting with the MDP $\hat{M}$.
We define the terminal value $V^{\hat{M}}_{\pi, H}(s) = 0$ for all $s\in \states$.
Define the expected value of a policy $\pi$ under an MDP $\hat{M}$ as $\widebar{V}^{\hat{M}}_{\pi} = \E{ V^{\hat{M}}_{\pi, 0}(s_0) \given \hat{M}, \pi }$.
We say a policy $\pi^{\hat{M}}$ is optimal for the MDP $\hat{M}$ if $\pi^{\hat{M}} \in \argmax_{\pi} V_{\pi, h}^{\hat{M}}(s) $ for all $(s, h) \in \states \times [H]$ where $[H] = \{0, 1, \ldots, H-1 \}$.
Let $\widebar{V}_{\pi} := \widebar{V}^{M}_{\pi}$ be expected value which measures the performance of a policy $\pi$ under true MDP $M$.

\paragraph{Observations and environmental randomness.}
The environmental randomness is carried out when nature samples the reward and next state given the current state-action pair.
During an episode $\ell$, the agent applies a policy $\pi_{\ell}$ and obtains the \emph{observations} up to stage $h$, which is
\(
O_{\ell, h} := (s_{\ell, 0}, a_{\ell,0}, s_{\ell, 1}, \ldots, s_{\ell, h-1}, a_{\ell, h-1}, s_{\ell, h} ).
\)
The \emph{observational history} before episode $\ell$ is denoted as $\ohist{\ell} = (O_{1, H}, \ldots, O_{\ell-1, H} ) $, representing all realized environmental randomness that the agent could extract information from.

\paragraph{Algorithm and algorithmic randomness.}
The agent's behavior is governed by a reinforcement learning algorithm $\operatorname{alg}$.
The algorithmic randomness may be introduced when the agent applies a randomized algorithm to select actions in a way that depends on internally generated random numbers.
Denote the random numbers used by $\operatorname{alg}$ in episode $\ell$ as $U_{\ell}$. The history of algorithmic randomness before episode $\ell$ is denoted as  $\uhist{\ell} = (U_{1}, \ldots, U_{\ell-1})$.
At the beginning of episode $\ell$, the algorithm produces a policy
\(
\pi_{\ell}=\operatorname{alg}\left(\states, \mathcal{A}, \ohist{\ell}, \uhist{\ell}, U_{\ell} \right)
\)
based on the state space $\states$ and action space $\actions$, the observational history $\ohist{\ell}$ before episode $\ell$ and possibly the history of algorithmic randomness $\uhist{\ell}$ before episode $\ell$ as well as the algorithmic randomness $U_{\ell}$ used in episode $\ell$.

Notice that the algorithmic randomness is a different source of randomness;
for each $k \in \N$, $U_{k}$ is jointly independent of $\{U_{\ell}\}_{\ell \neq k}$, the environmental randomness and true MDP $M$.
The algorithm $\operatorname{alg}$ usually computes some intermediate quantities in episode $\ell$ such as the virtual value functions $\hat{V}_{\ell, h}$.
Conditioned on $\ohist{\ell}$ and $\uhist{\ell}$, the policy $\pi_{\ell}$ and the value functions $\hat{V}_{\ell, h}$ at each stage $h \in [H]$ are random only through their dependence on $U_{\ell}$.

\paragraph{Bayesian RL and regret.}
We work with Bayesian reinforcement learning (RL) framework \citep{strens2000bayesian}. %
To deliver our conceptual idea in a clean way, w.l.o.g., we assume\footnote{Extending the algorithm and analysis in this paper to unknown stochastic rewards poses no real difficulty.}
the agent understands everything about the MDP but is uncertain about the underlying transition $\trans$.
The agent's initial knowledge and uncertainty about $\trans$ is encoded in a prior (representative) distribution $\prob(\trans \in \cdot)$.
Thus, $\trans$ is treated as a random variable in the agent's mind.
Therefore, the underlying MDP $M$ is also a random variable with prior distribution $\prob(M \in \cdot)$.
Let $\pi^{*} = \pi^M$ denote an optimal policy,
where $\pi^M \in \arg \max _{\pi} \widebar{V}_{\pi}$ is a function of the MDP $M$ and is thus also a random variable.

The \emph{regret} incurred in episode $\ell$ is the gap between the optimal value $\widebar{V}_{\pi^*}$ and the state value of $\operatorname{alg}$ under true MDP $M$, i.e. $\widebar{V}_{\pi^*} - \widebar{V}_{\pi_{\ell}}$.
The \emph{expected cumulative regret} over $L$ episodes in the environment with underlying MDP $M$ is
$\Re(M, \operatorname{alg}, L)=\sum_{\ell=1}^{L} \E[M, \mathrm{alg}]{ \widebar{V}_{\pi^*} - \widebar{V}_{\pi_{\ell}} \given M },$
where the expectation integrates over actions, state transitions, and any the algorithmic randomness used by $\operatorname{alg}$, while the MDP $M$ is fixed.
It is often useful to analyze the performance of an agent in terms of \emph{Bayesian regret}:
\begin{align*}
  \mathfrak{B} \Re (\operatorname{prior}, \operatorname{alg}, L)=\E[M \sim \operatorname{prior}]{\Re(M, \operatorname{alg}, L)},
\end{align*}
where the expectation is taken over the prior distribution $\operatorname{prior} := \prob(M \in \cdot)$ over the true MDP.

Posterior sampling algorithm for reinforcement learning (PSRL)\citep{strens2000bayesian} serves as a popular huristics to minimize Bayesian regret.
\begin{algorithm}[htbp]
  \caption{PSRL (episode $\ell$)}
  \label{algorithm:general-psrl}
  \begin{algorithmic}[1]
    \REQUIRE{Prior distribution $\prob( M \in \cdot )$ for underlying model $M$.}
    \STATE{Sample $\hat{M}_{\ell} \sim \prob( M \in \cdot \given \ohist{\ell} )$.}
    \STATE{Solve optimal policy $\pi_\ell = \pi^{\hat{M}_{\ell}}$ under $\hat{M}_\ell$.}
    \RETURN{$\pi_\ell$}
  \end{algorithmic}
\end{algorithm}

\paragraph{Linear mixture MDPs.}
We consider a class of MDPs called \emph{linear mixture MDPs}, where the transition probability kernel $\trans$ in the class $\Pc$ satisfying that for any $x = (h, s, a) \in \Xc$, $\trans(x) = \langle \theta^*_h, \phi(\cdot \given x) \rangle$, i.e., the linear mixture of a number of basis kernels \citep{ayoub2020model,modi2020sample,zhou2021nearly}. Here $\phi\left(s^{\prime} \mid x \right): \states \times \Xc \rightarrow \mathbb{R}^{d}$ is a feature mapping, a.k.a, basis transition kernels.
The set $ \Theta^* = ( \theta^{*}_0, \ldots, \theta^*_{H-1} )$ includes the underlying model parameters which are random variables that should induce \emph{proper transition probability kernel} $\langle \theta_h^*, \phi(\cdot \given x) \rangle$ when combined with features.

We sometimes write $\hat{M} := M(\hat{\Theta})$ as a mapping from a parameter set $\hat{\Theta}$ to a linear mixture MDP $\hat{M}$.
In the class of linear mixture MDPs, the unknown true MDP $M = M(\Theta^*)$ is random only through its dependence on $\Theta^*$. Therefore, the prior we consider here is over model parameters $\prob\left( \Theta^* \in \cdot \right) $.
The posterior sampling (line 2 in \cref{algorithm:general-psrl}) is implemented as first sampling the set $\hat{\Theta}_{\ell} \sim \prob( \Theta^* \in \cdot \given \hist{\ell})$ that contains $\hat{\Theta}_{\ell} = ( \hat{\theta}_{\ell, 0}, \ldots, \hat{\theta}_{\ell, H-1})$ and then constructing transition kernel $\hat{\trans}_{\ell}(x) = \langle \hat{\theta}_{\ell, h}, \phi(\cdot \given x) \rangle $ for all $h \in [H]$ as well as the corresponding MDP $\hat{M}_\ell = ( \states, \actions, \hat{\trans}_{\ell}, R, H, \rho )$.
To summarize, the posterior sampling algorithm for linear mixture MDPs is described in \cref{algorithm:linear-psrl}.

\begin{algorithm}
  \caption{PSRL in linear mixture MDPs (episode $\ell$)}
  \label{algorithm:linear-psrl}
  \begin{algorithmic}[1]
    \REQUIRE{Prior distribution $\prob( \Theta^* \in \cdot )$}
    \STATE{Sample $\hat{\Theta}_{\ell} = (\hat{\theta}_{\ell, 0}, \ldots, \hat{\theta}_{\ell, H-1}) \sim \prob( \Theta^* \in \cdot \given \ohist{\ell})$.}
    \STATE{Construct $\hat{\trans}_{\ell}(h, \cdot, \cdot) = \left \langle \hat{\theta}_{\ell, h}, \phi( \boldsymbol{\cdot} \given \cdot, \cdot) \right \rangle, \forall h \in [H]$.}
    \STATE{Solve optimal policy $\pi_\ell$ under the sampled MDP $\hat{M}_\ell = ( \states, \actions, \hat{\trans}_{\ell}, R, H, \rho )$.}
    \RETURN{$\pi_{\ell}$}
  \end{algorithmic}
\end{algorithm}

\begin{definition}[Value-correlated feature]
  \label{def:value-cor-feat}
  For any $x = (h, s, a) \in \Xc $, we define the \emph{value-correlated feature} induced by any bounded function $V$ and any fixed feature mapping $\phi(s' \given x): \states \times \Xc \rightarrow \R^d$,
  \(
  \phi_{V}(x) = \sum_{s' \in \states} \phi\left(s^{\prime} \mid  x \right) V\left(s' \right).
  \)
  Note that we do not assume finite state space and thus the summation can be replaced by integration if the state space is unbounded. The integration can be attained by an integration oracle.
\end{definition}
\begin{definition}[Covariance matrix of unknown model parameters under posterior distribution]
  \label{def:cov_model}
  We use the quantity \emph{posterior variance} of the unknown model parameters $\rmGamma_{\ell, h} := \var{\theta^*_h \given \ohist{\ell}}$
  as a uncertainty measurement.
\end{definition}
The covariance matrix serves as a uncertainty measurement that how much knowledge of the environment $\theta^*$ the agent is still not captured  given the history of observations $\ohist{\ell}$.

\section{Bayesian regret bound for PSRL}
\label{sec:regret_bound}
\label{sec:linear-psrl-analysis}

\begin{assumption}
  \label{asmp:feature}
  For any $x = (h, s, a) \in \Xc$, let $\phi\left(s' \mid x \right): \states \times \Xc \rightarrow \mathbb{R}^{d}$ be a feature mapping satisfying that for any bounded function $V: \states \rightarrow[0,1]$,
  $\left\|\phi_{V}(x)\right\|_{2} \leq 1$,
  where $\phi_V$ is the value-correlated feature associated with $V$ (see~\cref{def:value-cor-feat}).
\end{assumption}

\begin{assumption}
  \label{asmp:mutual-independence}
  Unknown model parameters $\theta^*_{0}, \ldots, \theta^*_{H-1} \in \R^d$ are mutually independent.
\end{assumption}

\begin{assumption}
  \label{asmp:bounded-norm}
  Unknown model parameters have bounded norm $ \norm{\theta_h^*}_2 \le B$ a.s. for all $h \in [H]$.
\end{assumption}

\begin{blockquote}
  \begin{fact}
    \label{fact:bounded}
    If the support of prior distribution is over \emph{norm-bounded model parameters} satisfying \cref{asmp:bounded-norm}, i.e., $\norm{\theta_h}_2 \le B$ a.s. for all $h \in [H]$, we have
    \begin{align*}
      \rmGamma_{1,h}
       & \preceq \E{ \norm{\theta - \E{\theta}}^2_2 } \mI                        \\
       & = \E{ \norm{\theta}_2^2  - \norm{\E{\theta}}_2^2 } \mI \preceq B^2 \mI.
    \end{align*}
  \end{fact}
\end{blockquote}
Note that $T = HL$ is the total interaction steps, we have the two parts of the results separately stated in the following.
\cref{thm:inhomo-bayesian-regret-0}is the first prior-dependent regret bound of PSRL with function approximation while \cref{rem:prior-free-0} is a improved prior-free bound.
\begin{blockquote}
  \begin{Theorem}[Prior-dependent analysis]
    \label{thm:inhomo-bayesian-regret-0}
    For any prior over models $\Theta^* = (\theta^*_0, \ldots, \theta^*_{H-1})$ satisfying \cref{asmp:mutual-independence}, PSRL have the Bayesian regret bound $\mathfrak{B}\Re(\operatorname{prior}, \operatorname{PSRL}, L)$ over $L$ episodes interaction with the time-inhomogeneous linear mixture MDP satisfying \cref{asmp:feature}, bounded by
    \begin{align*}
      \sqrt{ 2 d H^3 L
        \sum_{h=0}^{H-1} \log \det \left( \mI + L \rmGamma_{1, h} \right) },
    \end{align*}
    where $\rmGamma_{1, h}$ is the covariance of $\theta_h^*$ under $\operatorname{prior}$ distribution.
  \end{Theorem}
\end{blockquote}
We elucidate the relationship between regret and the prior distribution by examining the variance of the transition kernel. Specifically, if the agent possesses substantial knowledge about the environment—manifested as an informative prior with low variance—then it is poised to swiftly converge to the optimal policy, thereby minimizing incurred regret. This concept is empirically supported in Appendix E.6 \citep{li2022hyperdqn}, where an informative prior, derived from a pre-trained model using historical game frames, markedly boosts online exploration efficiency in benchmark deep RL problems. Our analysis on the dependency of regret on prior knowledge substantiates the imperative for employing informative priors.
The proof can be found in \cref{sec:sketch}.
\begin{blockquote}
  \begin{remark}[Prior-free bound]
    \label{rem:prior-free-0}
    If \cref{asmp:bounded-norm} further holds, by the \cref{fact:bounded}, we have the trivial bound $\log \det ( \mI + dL \rmGamma_{1, h}) \le d \log ( 1 + dLB^2 )$ which leads to  $\mathfrak{B}\Re(\operatorname{prior}, \operatorname{PSRL}, L)$ bounded by
    \begin{align*}
      \sqrt{ 2}d \sqrt{H^4 L \log ( 1 + LB^2)} = \mathcal{O}( d\sqrt{H^3 T \log L} ).
    \end{align*}
  \end{remark}
\end{blockquote}

\paragraph{Comparison with \citep*{osband2014model}.}
\citet{osband2014model} utilize the technique of eluder dimension and frequentist confidence set \citep{russo2014learning} to analyze PSRL with general function approximation. Their bound, when specified to linear mixture MDPs, is $\mathcal{O}(H^{3/2} d \sqrt{T} \log T)$.
Because of not relying on the union bound over confidence-sets to upper bound regret as they did, our final regret bound will be naturally improved by a $\sqrt{\log T}$ factor.
See more discussion of this $\log T$ term in \cref{sec:lower-bound-conjecture}.

\section{A value-targeted perspective of model learning}
\label{sec:value_pers}

\paragraph{Dynamic programming.}
Given the model $\hat{M}$, value iteration is a classical dynamic programming solver for optimal policy.
For any $x = (h, s, a) \in \Xc$, given the transition kernel $\hat{\trans}$ and next-state value function $\hat{V}_{h+1}$, we define a short operator form for the expected next-state value,
\begin{align*}
  \hat{\trans} \hat{V}_{h+1} (x) = \E[s'_x \sim \hat{\trans}(x)]{\hat{V}_{h+1} (s'_x)}.
\end{align*}
We use this short notation in the following \cref{algorithm:dpvi}.
\begin{algorithm}[htbp]
  \caption{Value iteration}
  \label{algorithm:dpvi}
  \begin{algorithmic}[1]
    \REQUIRE{$\hat{M} = ( \states, \actions, H, \hat{\trans}, \hat{R}, \rho )$ and init $\hat{V}_{H}(\cdot) = 0$.}
    \FOR{Stage $h= H-1, \ldots, 0$}
    \STATE{$\hat{Q}_{h}(\cdot, \cdot) = \hat{R}(h, \cdot, \cdot) + \hat{\trans} \hat{V}_{h+1} (h, \cdot, \cdot)$}
    \STATE{$\hat{V}_{h}(\cdot) = \max_{a} \hat{Q}_{h}(\cdot, a)$}
    \STATE{$\pi_{h}(\cdot) = \argmax_{a} \hat{Q}_{h}(\cdot, a)$}
    \ENDFOR
    \RETURN{Policy $\pi = (\pi_{h})_{h\in [H]}$ and value $(\hat{V}_{h})_{h \in [H]}$}
  \end{algorithmic}
\end{algorithm}

\paragraph{Model-based regret decomposition.}
\label{sec:generic-regret-decompose}
In episode $\ell$, a model-based algorithm typically construct a \emph{virtual MDP} $\hat{M}_\ell$ with \emph{virtual transition model} $\hat{\trans}_{\ell}$ based on the observational history $\ohist{\ell}$ and solve a (near)-optimal policy $\pi_\ell$ under $\hat{M}_{\ell}$. Then, the RL agent acts according to the policy $\pi_{\ell}$ in episode $\ell$.
The regret in episode $\ell$ can be decomposed to two terms,
\begin{align}
  \label{eq:regret-decompose}
  \barV{\pi^*} - \barV{\pi_\ell} =
  \underbrace{ \barV[M]{\pi^*} - \barV[\hat{M}_{\ell}]{\pi_{\ell}}  }_{ \text{Pessimism} }
  +
  \underbrace{ \barV[\hat{M}_{\ell}]{\pi_{\ell}} - \barV[M]{\pi_\ell} }_{\text{Estimation error}}
\end{align}

The \emph{pessimism} term is often made non-positive in expectation or with high probability via exploration mechanism induced by the algorithm.
Then we suffice to bound the Bayesian regret by the expected summation of estimation error $ \E{ \sum_{\ell=1}^L \barV[\hat{M}_{\ell}]{\pi_{\ell}} - \barV[M]{\pi_\ell} }.$
The \emph{estimation error} in episode $\ell$ is the error of estimating value $\barV[M]{\pi_{\ell}}$ of a particular policy $\pi_{\ell}$ via the virtual model $\hat{M}_{\ell}$.
Let $\Xc = [H] \times \states \times \actions$ and let $T = HL$ denote the total number of interactions in $L$ episodes.
Generally,
we can decompose the estimation error into the \emph{value-targeted model error} along the on-policy trajectory as shown in the following lemma~\citep{kearns2002near, osband2013more}. One can find the proof in \cref{sec:proof-simulation}.
\begin{lemma}
  \label{lem:simulation}
  Define the short notation $x_h := (h, s_h, a_h).$ For any MDP $\hat{M}$ and policy $\pi$,
  \begin{align*}
    \barV[\hat{M}]{\pi} - \barV[{M}]{\pi} = \sum_{h=0}^{H-1}
    \mathbb{E} \Big[ \underbrace{\left( \hat{\trans} - \trans \right) \V{\hat{M}}{\pi, h+1} (x_h) }_{\text{value-targeted model error}}\given M, \hat{M}, \pi \Big],
  \end{align*}
  where $\hat{\trans}$ is the transition model under $\hat{M}$ and $x_h$ is rolled out under the policy $\pi$ over the MDP $M$.
\end{lemma}

\subsection{Value-targeted model learning}
\label{sec:value-targeted-model}
As shown in \cref{lem:simulation} and \eqref{eq:regret-decompose}, the value-targeted model error is a sufficient condition for bounding estimation error and thus bounding the regret. Then, we conclude the following rule:
\begin{center}
  \emph{A virtual model is a desideratum for a task if it is accurate enough to help predict future high values.}
\end{center}
We can also treat the value-targeted model error as a tracking signal for the learning progress of a model-based RL algorithm.
Motivated by this observation,
through this work, we take the \emph{value-target perspective} for model-based learning. This perspective is related to the value-targeted regression \citep{ayoub2020model}.

Let $\Pc$ be a general transition function class.
Let $\mathcal{B}(\states, H)$ denotes the set of real-valued measurable functions with domain $\states$ that are bounded by $H$.
Let $\mathcal{V} \subseteq \mathcal{B}$ be the set of value functions.
For any tuple $ x = (h, s, a) \in \mathcal{X}$, any value function $V \in \mathcal{V}$, and any underlying transition model $\trans \in \Pc$, let the next-state be $s'_x \sim \trans(x)$ and let the value at sampled next-state be $Y = V(s'_x)$.
Let the expected value at next-state given the $x$ and $V$ be $f_{\trans}(x, V) = \E[s'_x \sim \trans(x)]{V(s'_x)}$.
Note that the input $(x, V)$ and transition model $\trans$ can be random variables.
The conditional expectation and variance of the outcome $Y$ given input become
\begin{align*}
  \E{Y \given \trans, x, V}   & = f_{\trans}(x, V),                     \\
  \var{Y \given \trans, x, V} & = \var[s'_x \sim \trans( x )]{V(s'_x)}.
\end{align*}
From the \emph{value-targeted perspective}, the \emph{information acquisition} of transition model can be made possible through the sequence of input $(x, V)$ and observed outcome $Y = V(s'_{x})$:
\begin{align}
  \label{eq:value-predict-pers}
  Y = f_{\trans}(x, V) + Z,
\end{align}
where $Z$ is the noise term with conditional variance $\var{Z \given \trans, x, V} = \var{Y \given \trans, x, V}$. The noise $Z$ involves the transitional noise when sampling next state and the value at the sampled state.

To explicitly represent the input and outcome sequentially, we define a new notion of history.
Denote $\hat{V}_{\ell, h} := \V{\hat{M}_{\ell}}{\pi_{\ell}, h}$ for $h \in [H]$ and let $\hat{V}_{\ell, H + 1} = 0$ and $a_{\ell, H} = a \in \actions$ be constant dummy variables. We are ready to define the \emph{value-augmented observations} up to stage $h$ in episode $\ell$,
\begin{align*}
  \hat{O}_{\ell, h}
   & := ( s_{\ell, 0}, a_{\ell, 0}, \hat{V}_{\ell, 1}, \ldots, s_{\ell, h}, a_{\ell, h}, \hat{V}_{\ell, h+1}) \\
   & = ( x_{\ell, 0}, \hat{V}_{\ell, 1}, \ldots, x_{\ell, h}, \hat{V}_{\ell, h+1})
\end{align*}
where the short notation is $x_{\ell, h} := (h, s_{\ell, h}, a_{\ell, h})$.
We define the \emph{history} up to the beginning of episode $\ell$ as
\[
  \hist{\ell} := ( \pi_{1}, \hat{O}_{1, H}, \ldots, \pi_{\ell-1}, \hat{O}_{\ell-1, H} )
\]
and the \emph{history} up to the stage $h$ of episode $\ell$ as
\[
  \hist{\ell, h} := ( \hist{\ell}, \pi_{\ell}, \hat{O}_{\ell, h}) .
\]
By definition, the input $(x_{\ell, h}, \hat{V}_{\ell, h+1})$ at stage $h$ of episode $\ell$ is $\sigma(\hist{\ell,h})$-measurable while the outcome $Y_{\ell, h+1} := \hat{V}_{\ell, h+1}(s_{\ell, h+1})$ is $\sigma(\hist{\ell, h+1})$-measurable. This is a sequential representation of the information process and will be useful throughout the presentation.

Similarly to \cref*{lem:simulation}, we have a generic regret decomposition from value-targeted perspective.
\begin{lemma}[Estimation decomposition conditioned on history]
  \label{lem:simulation-history}
  Denote the estimation error in stage $h$ of episode $\ell$ by $\Delta_{\ell, h}(s) = V^{\hat{M}_{\ell}}_{{\pi}_{\ell}, h}(s) - V^{M}_{{\pi}_{\ell}, h}(s)$ for any fixed state $s \in \states$, we can decompose $\E[\ell, h]{ \Delta_{\ell, h}(s_{\ell, h}) }$ as
  \begin{align*}
    \E[\ell, h]{ \sum_{j=h}^{H-1} \E[\ell, j]{ \left( \hat{\trans}_{\ell} - \trans \right) \hat{V}_{\ell, j+1}(x_{\ell, j} ) } }.
  \end{align*}
\end{lemma}

\subsection{PSRL from value-targeted perspective}
In this section, we state a key observation for analysis and show the posterior distribution for the unknown true MDP $M$ given the \emph{observational history} is the same as the posterior given the \emph{history}.
\begin{lemma}
  \label{lem:aug-posterior-equiv}
  For any $\sigma(\ohist{\ell}, \uhist{\ell}, U_{\ell})$-measurable random variable $\xi$, we have
  $\prob(M \in \cdot \given \ohist{\ell}) = \prob(M \in \cdot \given \xi, \ohist{\ell}).$
\end{lemma}
\cref{lem:aug-posterior-equiv} is due to the following key observations.
\begin{align*}
  \prob(M \in \cdot \given \ohist{\ell})
   & = \prob(M \in \cdot \given \ohist{\ell}, \uhist{\ell}, U_{\ell} )      \\
   & = \prob(M \in \cdot \given \xi, \ohist{\ell}, \uhist{\ell}, U_{\ell} ) \\
   & = \prob(M \in \cdot \given \xi, \ohist{\ell}).
\end{align*}
The first equality is due to the independence between $M \independent \uhist{\ell}$, $M \independent U_{\ell}$ and $U_{\ell} \independent \uhist{\ell}$.
The second equality is by definition $\xi$ is $\sigma(\ohist{\ell}, \uhist{\ell}, U_{\ell})$-measurable.
The third equality is also due to the independent algorithmic randomness.
For any historical episode $\ell' < \ell$, the historical policy $\pi_{\ell'}$ and the historical imagined value functions $\{\hat{V}_{\ell', h} \}_{h=0}^{H+1}$ are $\sigma(\ohist{\ell'}, \uhist{\ell'}, U_{\ell'})$-measurable by definition and thus $\sigma(\ohist{\ell}, \uhist{\ell}, U_{\ell})$-measurable.
Recall the difference between of $\ohist{\ell}$ and $\hist{\ell}$ is exactly the historical policies and historical imagined value functions. Therefore, by \cref{lem:aug-posterior-equiv}, the posterior distribution of true MDP $M$ conditioned on the observational history is equivalent to the one conditioned on history:
$\prob( M \in \cdot \given \ohist{\ell} ) = \prob( M \in \cdot \given \hist{\ell}).$
By posterior sampling, we have $\prob ( \hat{M}_{\ell} \in \cdot \given{ \ohist{\ell}} ) = \prob \left( M \in \cdot \given \ohist{\ell} \right)$. Therefore, for any $\ohist{\ell}$-measurable function $g$, we conclude $\prob (g(M) \in \cdot \given \ohist{\ell}) = \prob(g(M^\ell) \in \cdot \given \ohist{\ell})$.
This is noted by \citep{osband2013more,russo2014learning} and has an immediate consequence on the pessimism term:
\[
  \E{ \barV[M]{\pi^*} - \barV[\hat{M}_{\ell}]{\pi_\ell} \given[\Bigg] \hist{\ell} } = \E{ \barV[M]{\pi^*} - \barV[\hat{M}_{\ell}]{\pi_\ell} \given[\Bigg] \ohist{\ell} } = 0.
\]
Due to the posterior equivalence on MDP $M$ between $\ohist{\ell}$-conditioning and $\hist{\ell}$-conditioning and to promote the value-targeted perspective of model learning, we use $\hist{\ell}$-conditioning later in the presentation.
Also, for the purpose of simplicity, we sometimes use the following short notations
$\prob_{\ell, h}( \cdot), \E[\ell, h]{ \cdot}$ and $\var[\ell, h]{ \cdot}$ to donote conditional probability $\prob( \cdot \given \hist{\ell, h})$, conditional expectation $ \E{ \cdot \given \hist{\ell, h} } $ and conditional variance $\var{\cdot \given \hist{\ell, h}}$ respectively.

\subsection{Posterior variance reduction}
\label{sec:psrl-linear-value-pers}
In linear mixture MDPs, to mathematically quantify the reduction of uncertainty of the unknown environment when new information gathered in, we also relate the posterior variance of $\rmGamma_{\ell,h}$ in \cref{def:cov_model} to value targeted perspective.
First, we notice by \cref{lem:aug-posterior-equiv}
\begin{align*}
  \rmGamma_{\ell, h} = \var{ \theta^*_h \given \ohist{\ell} } = \var{ \theta^*_h \given \hist{\ell} }.
\end{align*}
Recall the definition of function $f_{\trans}$ in \cref{sec:value-targeted-model}.
Specifically, for the class of linear mixture MDPs, i.e. $\trans(x) = \langle \theta^*_h, \phi(\cdot \given x) \rangle$ for all $x=(h, s, a) \in \Xc$,
\begin{align}
  \label{eq:trans-linear-mixture}
  f_{\trans}(x, V)
   & = \trans V(x)
  = \sum_{s'_x \in \states} \left\langle \theta^*_h, \phi(s'_x \given x) \right\rangle V(s'_x)    \nonumber \\
   & = \langle \theta^*_h, \sum_{s'_x \in \states} \phi(s'_x \given x) V(s'_x) \rangle \nonumber            \\
   & = \langle \theta^*_h, \phi_V(x) \rangle .
\end{align}
Recall the relationship between the input $(x, V)$ and outcome $Y$ in \eqref{eq:value-predict-pers}, combined with \eqref{eq:trans-linear-mixture}, we have the following important observations,
\begin{align*}
  Y = \langle \theta^*_h, \phi_V(x) \rangle + Z
\end{align*}
is a noisy linear model with non-Gaussian prior over $\theta^*_h$ and non-Gaussian noise $Z$.
\begin{definition}
  \label{def:input_feature_history}
  Define the short notation $X_{\ell, h} := \phi_{\hat{V}_{\ell, h+1}}(x_{\ell, h}) = \phi_{\hat{V}_{\ell, h+1}}(h, s_{\ell, h}, a_{\ell, h})$.
\end{definition}
\begin{definition}
  \label{def:response_history}
  Define $Y_{\ell, h+1} := \hat{V}_{\ell, h+1}(s_{\ell, h+1})$.
\end{definition}
Recall the definition of value-augmented history $\hist{\ell}$ and $\hist{\ell, h}$ in \cref{sec:value_pers},
we can see $X_{\ell, h}$ is $\sigma(\hist{\ell, h})$-measurable since $X_{\ell, h}$ is a deterministic function of $(s_{\ell, h}, a_{\ell, h}, \hat{V}_{\ell, h+1})$ which is included in $\hist{\ell, h}$.
Similarly, since $Y_{\ell, h+1}$ is a deterministic function of $(\hat{V}_{\ell, h+1}, s_{\ell, h+1})$ which is included in $\hist{\ell, h+1}$, we conclude $Y_{\ell, h+1}$ is $\sigma(\hist{\ell, h+1})$-measurable.
Then with the fact that $s_{\ell, h+1} \sim \trans(x_{\ell, h})$, by \eqref{eq:trans-linear-mixture},
we can verify that,
\begin{align*}
  \E{ Y_{\ell, h+1} \given \Theta^*, \hist{\ell, h} } = \left\langle \theta^*_h, X_{\ell, h} \right\rangle
\end{align*}
and the variance of $Y_{\ell, h+1}$ conditioned on history $\hist{\ell, h}$ and true model $\Theta^*$ is defined as $\sigma_{\ell, h}^2$ which is
\begin{align*}
  \var{Y_{\ell, h+1} \mid \Theta^*, \hist{\ell, h} }
  = \var{\hat{V}_{\ell, h+1}(s_{\ell, h+1}) \given M, \hist{\ell, h}}
\end{align*}
By \cref{asmp:mutual-independence}, conditioned on history $\hist{\ell}$, the trajectory $(s_{\ell, 0}, a_{\ell, 0}, \ldots, s_{\ell, h}, a_{\ell, h})$ up to stage $h$ of episode $\ell$ is independent of true transition kernel $\trans(h, \cdot, \cdot)$ and thus independent of true model parameter $\theta^*_h$.
By definition, conditioned on history $\hist{\ell}$, the policy $\pi_{\ell}$ and imagined value functions $\{ \hat{V}_{\ell, h} \}_{h=1}^{H}$ are random only through its dependence on algorithmic randomness $U_{\ell}$, which is also independent of $\theta^*_h$. Therefore, we have the following important relationship on the posterior variance,
\begin{align*}
   & \var{\theta_h^* \given \hist{\ell, h}}                                                         = \var{\theta_h^* \given \hist{\ell}, \pi_{\ell}, \hat{O}_{\ell, h}} \\
   & = \var{\theta_h^* \given \hist{\ell}, \pi_{\ell}, s_{\ell, 0}, a_{\ell, 0}, \hat{V}_{\ell, 1}, \ldots, s_{\ell, h}, a_{\ell, h}, \hat{V}_{\ell, h+1}}               \\
   & = \var{\theta_h^* \given \hist{\ell} } =: \rmGamma_{\ell, h}
\end{align*}
Similarly, conditioned on history $\hist{\ell, h+1}$, the trajectory $(s_{\ell, h+2}, a_{\ell, h+2}, \ldots, s_{\ell, H-1}, a_{\ell, H-1}, s_{\ell, H})$ after step $h+1$ in episode $\ell$ is independent of $\theta^*_h$. Then, we have,
\begin{align*}
  \var{\theta_h^* \given \hist{\ell, h+1} }
   & = \var{\theta_h^* \given \hist{\ell}, \pi_{\ell}, \hat{O}_{\ell, H}} \\
   & = \var{\theta_h^* \given \hist{\ell + 1} } =: \rmGamma_{\ell+1, h}
\end{align*}
Now we are ready to apply \cref{lem:vr-posterior} (posterior variance reduction lemma for general prior and general likelihood in noisy linear model) to the above noise linear system and obtain the \cref{thm:vd}.
It shows that, in expectation, the posterior covariance of the true model parameters will be reduced.
The reduction depends on the variance in the direction of the value-correlated feature vector quantified with the prior distribution of true model parameters.
The reduction also depends on the variance of value due to the transitional noise.
Denote short notations $\sigma^2_{\ell, h}$ $ := \var[\ell, h]{ \hat{V}_{\ell, h+1}(s_{\ell, h+1}) \given M } $ and  $X_{\ell, h} := \phi_{\hat{V}_{\ell, h+1}}(x_{\ell, h})$.
\begin{blockquote}
  \begin{Theorem}[Posterior variance reduction]
    \label{thm:vd}
    \begin{align*}
      \E[\ell, h]{ \rmGamma_{\ell+1, h} } \preceq \rmGamma_{\ell, h} - \frac{ \rmGamma_{\ell, h} X_{\ell, h} X_{\ell, h}^{\top} \rmGamma_{\ell, h} }{ \E[\ell, h]{{\sigma}_{\ell, h}^2} + X_{\ell, h}^{\top} \rmGamma_{\ell, h} X_{\ell, h}  },
    \end{align*}
    where $A \preceq B$ for any two matrices $A$ and $B$ iff $B - A$ is a positive semi-definite matrix.
  \end{Theorem}
\end{blockquote}

\begin{remark}
  \label{rem:sigmabar}
  For any $\bar{\sigma}_{\ell, h}$ such that $\bar{\sigma}^2_{\ell, h} \ge \E[\ell, h]{\sigma^2_{\ell, h}}$ almost surely, \cref{thm:vd} implies
  \begin{align*}
    \E[\ell, h]{ \rmGamma_{\ell+1, h} } \preceq \rmGamma_{\ell, h} - \frac{ \rmGamma_{\ell, h} X_{\ell, h} X_{\ell, h}^{\top} \rmGamma_{\ell, h} }{ \bar{\sigma}_{\ell, h}^2 + X_{\ell, h}^{\top} \rmGamma_{\ell, h} X_{\ell, h}  }.
  \end{align*}
  Further, by Sherman-Morrison formula, \cref{thm:vd} implies (assume invertible)
  \begin{align*}
    \E[\ell, h]{ \rmGamma_{\ell+1, h} }^{-1} \succeq \rmGamma_{\ell, h}^{-1} + \frac{1}{\bar{\sigma}^2_{\ell, h}}X_{\ell, h} X_{\ell, h}^\top.
  \end{align*}
  In the analysis, we choose $\bar{\sigma}_{\ell, h}$ s.t. $\bar{\sigma}^2_{\ell, h} := \E[\ell, h]{\sigma_{\ell, h}^2} \vee \sigma^2_{\min}$.
  We choose ${\sigma}_{\min} = H$.
  Note that the \cref{thm:vd} still holds if the parameter $\hat{\theta}_{\ell, h}$ does not induce a proper transition kernel.
\end{remark}

\subsection{Decouple regret to posterior variance}
\label{sec:regret-variance}
We are now ready to see how the regret is being translated to the posterior variance over models from the value-targeted perspective. All details of analysis can be found in \cref{sec:regret_decompose}.
Denote the estimation error in stage $h$ of episode $\ell$ shortly by $\Delta_{\ell, h}(s) = V^{\hat{M}_{\ell}}_{{\pi}_{\ell}, h}(s) - V^{M}_{{\pi}_{\ell}, h}(s)$ for any fixed state $s \in \states$.
The first step is to translated regret into inner product between model parameters and value-correlated features under linear mixture model assumptions.
\begin{lemma}[Regret to model posterior deviation over feature directions]
  \label{lem:linear-decomposition-0}
  \begin{align*}
    \E[\ell, h]{ \abs{ \Delta_{\ell, h}(s_{\ell, h}) } }
     & \le \E[\ell, h]{ \sum_{j=h}^{H-1} \E[\ell, j]{ \abs{\left\langle \hat{\theta}_{\ell, j} - \theta^*_j, X_{\ell, j} \right\rangle } } }.
  \end{align*}
\end{lemma}
The key to relate regret to posterior variance is the decoupling argument.
\begin{blockquote}
  \begin{lemma}[Decoupling lemma]
    \label{lem:linear-decoupling-0}
    Let $\theta^*$, $\hat{\theta}$ and $\phi$ be arbitrary random vectors in $\R^d$. If $\hat{\theta}$ is an i.i.d copy of $\theta^*$, then
    \begin{align*}
      \E{ \abs{ \left\langle \hat{\theta} - \theta^*, \phi \right\rangle } }^2
       & \le 2 d \E{ \phi^{\top} \var{\theta^*} \phi }.
    \end{align*}
    Note that the above inequality holds even when $\phi \in \R^d$ is dependent on $\theta^*$ or $\hat{\theta}$.
  \end{lemma}
\end{blockquote}
With \cref{lem:linear-decomposition-0,lem:linear-decoupling-0} in hand, we now give the general regret decomposition that is essential for prior-dependent Bayesian analysis of PSRL.
\begin{proposition}[General regret bound (Informal)]
  \label{prop:regret-decompose-0}
  For any episode $\ell$ and any stage $h$, $\E{ \sum_{\ell=1}^L \abs{ \Delta_{\ell, h}(s_{\ell, h}) } }$ is bounded by a term including the sum of posterior variance in the on-policy trajectory,
  \begin{align*}
    \sqrt{ d \E{ \sum_{\ell = 1}^{L} \sum_{h=0}^{H-1} \norm{ {X_{\ell, h}}}^2_{\rmGamma_{\ell, h} } } }.
  \end{align*}
  Notice that this relationship does not require any assumption additional to linear mixture MDPs.
\end{proposition}
With the general regret bound and the posterior variance reduction argument to track the learning process, we are ready to prove \cref{thm:inhomo-bayesian-regret-0}. The sketch and detail of the proof can be found in \cref{sec:sketch}.

\section{Discussions}
\label{sec:discuss}
We now discuss the \textbf{generality of the techniques} we developed and their possible applications for analyzing other problems:
\begin{itemize}
    \item The posterior variance reduction technique in \cref{sec:psrl-linear-value-pers} and \cref{sec:posterior-vr-general} is applicable to any prior-likelihood model as well as any learning algorithm.
          It measures the posterior variance reduction of model parameters only through features in the history (see the comparison between \cref{rem:posterior-vr,rem:gaussian}) which in some sense recover the property of Gaussian posterior.
          This key observation may also enable mis-specification analysis of Gaussian approximation of environment that is not Gaussian.
    \item The ability of posterior variance reduction lemma to characterize heterogeneous noise likelihood in a predictable way for the sequential decision problem may enable some rigorous understanding of the benefit of information directed sampling compared to thompson sampling.
    \item The decoupling technique~\cref{lem:linear-decoupling-0,lem:linear-decomposition} is possibly applicable to other \emph{randomized} algorithms when the imagined virtual model $\hat{\theta}$ is not exact identically distributed as the true model $\theta^*$.
\end{itemize}

\subsubsection*{Acknowledgements}
The work of Y. Li was supported by the presidential Ph.D. fellowship and SRIBD Ph.D. fellowship.
The work of Z.-Q. Luo was supported by the Guangdong Major Project of  Basic and Applied Basic Research (No.2023B0303000001), the Guangdong Provincial Key Laboratory of Big Data Computing, and the National Key Research and Development Project under grant 2022YFA1003900.

\bibliography{ref}

\section*{Checklist}

\begin{enumerate}

  \item For all models and algorithms presented, check if you include:
        \begin{enumerate}
          \item A clear description of the mathematical setting, assumptions, algorithm, and/or model. [Yes]
          \item An analysis of the properties and complexity (time, space, sample size) of any algorithm. [Yes]
          \item (Optional) Anonymized source code, with specification of all dependencies, including external libraries. [Not Applicable]
        \end{enumerate}

  \item For any theoretical claim, check if you include:
        \begin{enumerate}
          \item Statements of the full set of assumptions of all theoretical results. [Yes]
          \item Complete proofs of all theoretical results. [Yes]
          \item Clear explanations of any assumptions. [Yes]
        \end{enumerate}

  \item For all figures and tables that present empirical results, check if you include:
        \begin{enumerate}
          \item The code, data, and instructions needed to reproduce the main experimental results (either in the supplemental material or as a URL). [Not Applicable]
          \item All the training details (e.g., data splits, hyperparameters, how they were chosen). [Yes/No/Not Applicable]
          \item A clear definition of the specific measure or statistics and error bars (e.g., with respect to the random seed after running experiments multiple times). [Not Applicable]
          \item A description of the computing infrastructure used. (e.g., type of GPUs, internal cluster, or cloud provider). [Not Applicable]
        \end{enumerate}

  \item If you are using existing assets (e.g., code, data, models) or curating/releasing new assets, check if you include:
        \begin{enumerate}
          \item Citations of the creator If your work uses existing assets. [Not Applicable]
          \item The license information of the assets, if applicable. [Not Applicable]
          \item New assets either in the supplemental material or as a URL, if applicable. [Not Applicable]
          \item Information about consent from data providers/curators. [Not Applicable]
          \item Discussion of sensible content if applicable, e.g., personally identifiable information or offensive content. [Not Applicable]
        \end{enumerate}

  \item If you used crowdsourcing or conducted research with human subjects, check if you include:
        \begin{enumerate}
          \item The full text of instructions given to participants and screenshots. [Not Applicable]
          \item Descriptions of potential participant risks, with links to Institutional Review Board (IRB) approvals if applicable. [Not Applicable]
          \item The estimated hourly wage paid to participants and the total amount spent on participant compensation. [Not Applicable]
        \end{enumerate}

\end{enumerate}

\newpage
\appendix
\onecolumn

\section{Notations}
\label{sec:notation}
\label{sec:symbols}
In this section, we collect all appeared symbols in the \cref{tab:MainNotation}.
\newcommand{\defeq}{\vcentcolon=}
{
\renewcommand{\arraystretch}{1.5}
\begin{longtable}{l l}
\caption{Symbols}\\
\hline
\hline
Symbol & Meaning \\
\hline
$M$ &  Underlying true MDP for the environment. \\
$\Theta^*$ & The set of model parameters for underlying linear mixture MDP $\Theta^* = (\theta^*_0, \ldots, \theta^*_{H-1})$. \\
$\hat{M}_{\ell}$ & Virtual MDP constructed by the agent in episode $\ell$. \\
$\hat{\Theta}_{\ell} $ & The set of model parameters $\hat{\Theta}_{\ell} = (\hat{\theta}_{\ell, 0}, \ldots, \hat{\theta}_{\ell, H-1}) $ in episode $\ell$ sampled by PSRL. \\
$ s_{\ell, h} $ & state encountered in stage $h$ of episode $\ell$ by the agent. \\
$ a_{\ell, h} $ & action encountered in stage $h$ of episode $\ell$ by the agent. \\
$ x_{\ell, h} $ & short notation for $(h, s_{\ell, h}, a_{\ell, h})$. \\
$\hat{V}_{\ell, h}$ & the virtual value function $V_{\pi_{\ell}, h}^{\hat{M}_{\ell}}$ at stage $h$ computed in episode $\ell$. \\
$X_{\ell, h}$ & future-virtual-value-correlated feature $\phi_{\hat{V}_{\ell, h+1}}( x_{\ell, h} ) = \sum_{s'} \phi( s' \given x_{\ell, h}) \hat{V}_{\ell, h+1}(s')$. \\
$Y_{\ell, h}$ & virtual value  $\hat{V}_{\ell, h}(s_{\ell, h})$ at the state encountered in stage $h$ and episode $\ell$. \\
${\Delta}_{\ell, h}(s)$ & Estimation error of policy $\pi_{\ell}$ using virtual MDP $\hat{M}_{\ell}$, i.e. $V^{\hat{M}_{\ell}}_{{\pi}_{\ell}, h}(s) - V^{M}_{{\pi}_{\ell}, h}(s)$. \\
\hline
$O_{\ell, h}$ & Observations $(s_{\ell, 0}, a_{\ell, 0}, \ldots, s_{\ell, h}, a_{\ell, h})$. \\
$\ohist{\ell}$ & Observational history $\ohist{\ell} = ( O_{1, H}, \ldots, O_{\ell, H} )$ up to the beginning of episode $\ell$. \\
$\ohist{\ell, h}$ & Observational history $\ohist{\ell} = (\ohist{\ell}, O_{\ell, h} )$ up to the stage $h$ of episode $\ell$. \\
$\hat{O}_{\ell, h}$ & Value-augmented observations $(s_{\ell, 0}, a_{\ell, 0}, \hat{V}_{\ell, 1}, \ldots, s_{\ell, h}, a_{\ell, h}, \hat{V}_{\ell, h+1})$. \\
$\hist{\ell}$ & History $(\pi_1, \hat{O}_{1, H}, \ldots, \pi_{\ell-1}, \hat{O}_{\ell-1, H})$ up to the beginning of episode $\ell$. \\
$\hist{\ell, h}$ & History $(\hist{\ell}, \pi_{\ell}, \hat{O}_{\ell, h})$ up to the stage $h$ of episode $\ell$. Note $\hist{\ell, H}$ is exactly $\hist{\ell+1}$ \\
$\prob_{\ell}( \cdot)$ & short notation for conditional probability $ \prob( \cdot \given \hist{\ell})$. \\
$\prob_{\ell, h}( \cdot)$ & short notation for conditional probability $\prob( \cdot \given \hist{\ell, h})$. \\
$\E[\ell]{ \cdot}$ & short notation for conditional expectation $ \E{ \cdot \given \hist{\ell} } $. \\
$\E[\ell, h]{ \cdot}$ & short notation for conditional expectation $ \E{ \cdot \given \hist{\ell, h} } $. \\
$\var[\ell]{ \cdot}$ & short notation for conditional variance (covariance) $ \var{ \cdot \given \hist{\ell} } $. \\
$\var[\ell, h]{ \cdot}$ & short notation for conditional variance (covariance) $ \var{ \cdot \given \hist{\ell, h} } $. \\ 
$ \E{\cdot}^{p}$ & The $p$-th power to the expectation $\left( \E{\cdot} \right)^{p}$. \\
$\rmGamma_{\ell, h}$ & the posterior covariance matrix $\var[\ell]{\theta^*_h}$ of the true parameter $\theta^*_h$. \\
$\sigma^2_{\ell, h}$ & virtual-value variance only due to transitional noise $ \var[\ell,h]{ \hat{V}_{\ell, h+1}(s_{\ell, h+1}) \given M } $.  \\
$ \sigma_{\min}$ & a deterministic constant for the convenient of analysis. \\ 
$\bar{\sigma}_{\ell, h}$ & random variable $ \sqrt{\E[\ell, h]{\sigma_{\ell, h}^2} } \vee \sigma_{\min} $. \\
\label{tab:MainNotation}
\end{longtable}
}

\section{Regret decomposition}
\label{sec:regret_decompose}
Define the estimation error in stage $h$ of episode $\ell$ by $\Delta_{\ell, h}(s) = V^{\hat{M}_{\ell}}_{{\pi}_{\ell}, h}(s) - V^{M}_{{\pi}_{\ell}, h}(s)$ for any fixed state $s \in \states$.
Recall the definition of $X_{\ell, h}$ in \cref{def:input_feature_history} and $\bar{\sigma}_{\ell, h}$ in \cref{rem:sigmabar}, we have the regret decomposition
\begin{proposition}[Regret decomposition]
  \label{prop:regret-decompose}
  For any episode $\ell$ and any stage $h$,
  \begin{align*}
    \E{ \sum_{\ell=1}^L \abs{ \Delta_{\ell, h}(s_{\ell, h}) } }
    \le \sqrt{ 2 d \E{ \sum_{\ell = 1}^{L} \sum_{h=0}^{H-1} \bar{\sigma}_{\ell, h}^2 } \E{ \sum_{\ell = 1}^{L} \sum_{h=0}^{H-1} 1 \wedge \norm{ \frac{X_{\ell, h}}{\bar{\sigma}_{\ell, h}}}^2_{\rmGamma_{\ell, h} } } }
  \end{align*}
  Furthermore, if \cref{asmp:mutual-independence} holds,
  \begin{align*}
    \E{ \sum_{\ell=1}^L \barV[\hat{M}_{\ell}]{\pi_{\ell}} - \barV[M]{\pi_{\ell}} } \le \sqrt{d \E{ \sum_{\ell = 1}^{L} \sum_{h=0}^{H-1} \bar{\sigma}_{\ell, h}^2 } \E{ \sum_{\ell = 1}^{L} \sum_{h=0}^{H-1} 1 \wedge \norm{ \frac{X_{\ell, h}}{\bar{\sigma}_{\ell, h}}}^2_{\rmGamma_{\ell, h} } } }
  \end{align*}
\end{proposition}

\subsection{Regret decomposition}
\label{sec:appendix-decompose}
Recall the short notation of estimation error in stage $h$ of episode $\ell$, $\Delta_{\ell, h}(s) = V^{\hat{M}_{\ell}}_{{\pi}_{\ell}, h}(s) - V^{M}_{{\pi}_{\ell}, h}(s)$. We have the following lemmas prepared for \cref{prop:regret-decompose}.
\begin{lemma}[Regret to model posterior deviation over feature directions]
  \label{lem:linear-decomposition}
  \begin{align*}
    \E[\ell, h]{ \Delta_{\ell, h}(s_{\ell, h}) }
     & = \E[\ell, h]{ \sum_{j=h}^{H-1} \E[\ell, j]{ \left\langle \hat{\theta}_{\ell, j} - \theta^*_j, X_{\ell, j} \right\rangle } }          \\
    \E[\ell, h]{ \abs{ \Delta_{\ell, h}(s_{\ell, h}) } }
     & \le \E[\ell, h]{ \sum_{j=h}^{H-1} \E[\ell, j]{ \abs{\left\langle \hat{\theta}_{\ell, j} - \theta^*_j, X_{\ell, j} \right\rangle } } }
  \end{align*}
\end{lemma}

\begin{corollary}
  \label{cor:linear-decompose}
  If \cref{asmp:mutual-independence} (mutual-independence of unknown model parameters) holds, a corollary of \cref{lem:linear-decomposition} is
  \begin{align*}
    \E[\ell, h]{ \Delta_{\ell, h}(s_{\ell, h}) }
    = \E[\ell, h]{ \sum_{j=h}^{H-1} \E[\ell, j]{ \left\langle \hat{\theta}_{\ell, j} - \E[\ell]{\hat{\theta}_{\ell, j}}, X_{\ell, j} \right\rangle } }
  \end{align*}
\end{corollary}

\subsection{Decoupling}
In order to bound the RHS of \cref{lem:linear-decomposition}, we utilize the idea in \citep{russo2016information, kalkanli2020improved} for linear bandit.
The \cref{lem:linear-decoupling} shown here is slightly stronger than previous literature. This stronger statement is essential to bound the absolute value-targeted error in \cref{prop:regret-decompose} which is required for proving a better $H$ dependence by bounding the absolute estimation error $\abs{\Delta_{\ell, h}}$ in \cref{prop:sum-virtual-value-variance}.
Extending previous statement in bandit literature to this stronger statement is not difficult but new. The proof of \cref{lem:linear-decoupling} can be found in \cref{sec:tech-proof-linear-decoupling}.
\begin{lemma}
  \label{lem:linear-decoupling}
  Let $X_{1}$ and $X_{2}$ be arbitrary i.i.d., $\mathbb{R}^{m}$ valued random variables and $f_{1}, f_{2}$ measurable maps such that $f_{1}, f_{2}$ :
  $\mathbb{R}^{m} \rightarrow \mathbb{R}^{d}$ with $\E{ \norm{ f_{1}\left(X_{1}\right) }_{2}^{2} }, \E{ \norm{ f_{2}\left(X_{1}\right) }_{2}^{2} } <\infty,$ then
  \begin{align*}
    \E{ \abs{ f_{1}\left(X_{1}\right)^{\top} f_{2}\left(X_{1}\right) } }^2
    \le d \E{  \left(f_{1}\left(X_{1}\right)^{\top} f_{2}\left(X_{2}\right)\right)^{2} }.
  \end{align*}
  Specifically, let $X_1 = (X, Z)$, $f_1(x, z) = x, f_2(x, z) = z$, and $X_2$ be a i.i.d copy of $X_1$,
  \begin{align*}
    \E{ \abs{X^{\top} Z } }^{2}
    \leq d
    \Tr\left( { \E{ X X^{\top} } } { \E{ Z Z^{\top} } } \right)
  \end{align*}
\end{lemma}
\begin{corollary}
  \label{cor:linear-decoupling}
  Let $\theta^*$, $\hat{\theta}$ and $\phi$ be arbitrary random vectors in $\R^d$. If $\hat{\theta}$ is an i.i.d copy of $\theta^*$, then
  \begin{align*}
    \E{ \abs{ \left\langle \hat{\theta} - \E{ \hat{\theta} }, \phi \right\rangle } }^2
     & \le d \E{ \phi^{\top} \var{ \theta^*} \phi }   \\
    \E{ \abs{ \left\langle \hat{\theta} - \theta^*, \phi \right\rangle } }^2
     & \le 2 d \E{ \phi^{\top} \var{\theta^*} \phi }.
  \end{align*}
  Note that the above inequality holds even when $\phi \in \R^d$ is dependent on $\theta^*$ or $\hat{\theta}$.
\end{corollary}
\begin{proof}
  Notice the \cref{lem:linear-decoupling} do not require the independence between $X$ and $Z$. Therefore, we can apply \cref{lem:linear-decoupling} with setting $X = \hat{\theta} - \E{\hat{\theta}}$ and $Z = \phi$ to yield
  \begin{align*}
    \E{ \abs{ \left\langle \hat{\theta} - \E{ \hat{\theta}}, \phi \right\rangle } }^2
    \le d \Tr\left( \var{ \hat{\theta} } \E{ \phi \phi^{\top}} \right)
    \stackrel{(b)}{=} d \E{ \phi^{\top} \var{ {\theta}^*} \phi}
  \end{align*}
  where $(b)$ is from $\var{\theta^*} = \var{\hat{\theta}}$ by the i.i.d condition and the linearity of expectation and trace operator.

  Similarly, to upper bound the RHS of the second inequality in the corollary, we can apply \cref{lem:linear-decoupling} with setting $X = \hat{\theta} - \theta^*$ and $Z = \phi$.
  Notice that by i.i.d condition between $\theta^*$ and $\hat{\theta}$,
  \begin{align*}
    \E{ \left( \theta^* - \hat{\theta} \right) \left( \theta^* - \hat{\theta} \right)^\top }
     & = \E{ \left( \theta^*  - \E{\theta^*} + \E{\hat{\theta}} - \hat{\theta} \right)  \left( \theta^*  - \E{\theta^*} + \E{\hat{\theta}} - \hat{\theta} \right)^\top } \\
     & = \var{\theta^* - \E{\theta^*}} + \var{\hat{\theta} - \E{\hat{\theta}} }
    = 2 \var{\theta^*}
  \end{align*}
  Then we derive,
  $
    \E{ \abs{ \left\langle \hat{\theta} - \theta^*, \phi \right\rangle } }^2
    \le 2 d \E{ \phi^{\top} \var{\theta^*} \phi}.
  $
\end{proof}

\subsection{Proof of \cref{prop:regret-decompose}}
\begin{proof}
  Recall the definition of $\hist{\ell, h} = (\hist{\ell}, \pi_{\ell}, \hat{O}_{\ell, h})$, we have
  \begin{align}
    \label{eq:abs-estimation-error-conditional-expect}
    \E{ \abs{ \Delta_{\ell, h}(s_{\ell, h}) } \given \hist{\ell} }
     & = \E{ \E{ \abs{ \Delta_{\ell, h}(s_{\ell, h}) }  \given \hist{\ell, h}} \given \hist{\ell} } = \E[\ell]{ \E[\ell, h]{ \abs{ \Delta_{\ell, h}(s_{\ell, h}) } } }
  \end{align}
  With application of \cref{lem:linear-decomposition}, we have
  \begin{align}
    \E[\ell]{ \E[\ell, h]{ \abs{ \Delta_{\ell, h}(s_{\ell, h}) } } }
     & \le \E[\ell]{ \E[\ell, h]{ \sum_{j=h}^{H-1} \E[\ell, j]{ \abs{\left\langle \hat{\theta}_{\ell, j} - \theta^*_j, X_{\ell, j} \right\rangle } } } }  = \sum_{j=h}^{H-1} \E[\ell]{ \abs{\left\langle \hat{\theta}_{\ell, j} - \theta^*_j, X_{\ell, j} \right\rangle } }.
    \label{eq:abs-decompose-0}
  \end{align}
  Combining \eqref{eq:abs-estimation-error-conditional-expect} and \eqref{eq:abs-decompose-0} yields,
  \begin{align}
    \label{eq:abs-decompose-1}
    \E[\ell]{ \abs{ \Delta_{\ell, h}(s_{\ell, h}) } } \le \sum_{j=h}^{H-1} \E[\ell]{ \abs{\left\langle \hat{\theta}_{\ell, j} - \theta^*_j, X_{\ell, j} \right\rangle } }
  \end{align}
  Recall the definition $\Delta_{\ell, 0}(s) = V^{\hat{M}_{\ell}}_{\pi_{\ell}, h}(s) - V^{M}_{\pi_{\ell}, h}(s)$. Since the initial state $s_{\ell, 0} \sim \rho$, we have
  $\barV[\hat{M}_{\ell}]{\pi_{\ell}} - \barV[M]{\pi_{\ell}} = \E{ \Delta_{\ell, 0}(s_{\ell, 0}) \given \hat{M}_{\ell}, M, \pi_{\ell} }$.
  It follows that,
  \begin{align}
    \label{eq:initial-state-expectation}
    \E[\ell]{ \barV[\hat{M}_{\ell}]{\pi_{\ell}} - \barV[M]{\pi_{\ell}} }
     & = \E[\ell]{ \Delta_{\ell, 0}(s_{\ell, 0}) }
    = \E[\ell]{ \E[\ell, 0]{ \Delta_{\ell, 0}(s_{\ell, 0}) } }
  \end{align}
  By applying \cref{lem:linear-decomposition},
  \begin{align}
    \E[\ell]{ \E[\ell, 0]{ \Delta_{\ell, 0}(s_{\ell, 0}) } } & = \E[\ell]{ \E[\ell, 0]{ \sum_{j=0}^{H-1} \E[\ell, j]{ \left\langle \hat{\theta}_{\ell, j} - \theta^*_j, X_{\ell, j} \right\rangle } } }
    = \sum_{h=0}^{H-1} \E[\ell]{  \left\langle \hat{\theta}_{\ell, h} - \theta^*_h, X_{\ell, h} \right\rangle }
    \label{eq:decompose-0}
  \end{align}
  If \cref{asmp:mutual-independence} holds, by Corollary~\ref{cor:linear-decompose}, we can further derive,
  \begin{align}
    \E[\ell]{ \E[\ell, 0]{ \Delta_{\ell, 0}(s_{\ell, 0}) } } & = \E[\ell]{ \E[\ell, 0]{ \sum_{j=0}^{H-1} \E[\ell, j]{ \left\langle \hat{\theta}_{\ell, j} - \E[\ell]{\hat{\theta}_{\ell, j}}, X_{\ell, j} \right\rangle } } }
    = \sum_{h=0}^{H-1} \E[\ell]{  \left\langle \hat{\theta}_{\ell, h} - \E[\ell]{\hat{\theta}_{\ell, h}}, X_{\ell, h} \right\rangle }
    \label{eq:decompose-1}
  \end{align}
  Combining \eqref{eq:initial-state-expectation} and \eqref{eq:decompose-1} yields,
  \begin{align}
    \E[\ell]{ \barV[\hat{M}_{\ell}]{\pi_{\ell}} - \barV[M]{\pi_{\ell}} }
    = \sum_{h=0}^{H-1} \E[\ell]{  \left\langle \hat{\theta}_{\ell, h} - \E[\ell]{\hat{\theta}_{\ell, h}}, X_{\ell, h} \right\rangle }
    \label{eq:decompose-2}
  \end{align}
  Notice that $\hat{\theta}_{\ell, h}$ is an i.i.d copy of $\theta^*_h$ conditioned on history $\hist{\ell}$ by the algorithm of posterior sampling.
  Recall the definition of $\rmGamma_{\ell, h} = \var[\ell]{\theta^*_h}$, with Corollary~\ref{cor:linear-decoupling},
  we have
  \begin{align}
    \label{eq:decouple-0}
    \E[\ell]{ \abs{ \left \langle \hat{\theta}_{\ell, h} - \E[\ell]{ \hat{\theta}_{\ell, h} }, X_{\ell, h} \right \rangle } }^2
     & \le d \E[\ell]{ X_{\ell, h}^{\top} \rmGamma_{\ell, h} X_{\ell, h}}    \\
    \label{eq:decouple-1}
    \E[\ell]{ \abs{ \left\langle \hat{\theta}_{\ell, h} - \theta^*_h, X_{\ell, h} \right\rangle } }^2
     & \le 2 d \E[\ell]{ X_{\ell, h}^{\top} \rmGamma_{\ell, h} X_{\ell, h}}.
  \end{align}
  The RHS of the above inequality is possibly larger than the valid range of LHS. Therefore, we try to tighten the upper bound by clipping.
  First, note that for any $x= (h, s, a)\in \mathcal{X}$, any MDP $\hat{M}$ and any policy $\pi$, we have the following equality holds
  \begin{align*}
    \abs{ (\hat{\trans} - \trans ) V^{\hat{M}}_{\pi, h+1}(x) } = \abs{ \hat{\trans} V^{\hat{M}}_{\pi,h+1}(x) -\trans V^{\hat{M}}_{\pi,h+1}(x) } \le H \quad a.s.
  \end{align*}
  Recall that $(\hat{\theta}_{\ell, h} - \theta^*_h)^\top X_{\ell, h} = ( \hat{\trans} - \trans)\hat{V}_{\ell, h}(x_{\ell, h})\le H$. Therefore, for any constant $\sigma_{\min} \ge H/\sqrt{d}$ and any random variable $\bar{\sigma}_{\ell, h} \ge \sigma_{\min}$ almost surely,
  \begin{align}
    \label{eq:one-episode-upper-bound}
    \E[\ell]{ \abs{ \left\langle \hat{\theta}_{\ell, h} - \E[\ell]{ \hat{\theta}_{\ell, h}}, X_{\ell, h} \right\rangle } }
     & \le \sqrt{d} \E[\ell]{ \bar{\sigma}_{\ell, h}^2 \cdot 1 \wedge (X_{\ell, h}/\bar{\sigma}_{\ell, h} )^\top {\rmGamma_{\ell, h}} (X_{\ell, h}/\bar{\sigma}_{\ell, h} ) }^{1/2}  \\
    \label{eq:one-episode-upper-bound-abs}
    \E[\ell]{ \abs{ \left\langle \hat{\theta}_{\ell, h} - \theta^*_h, X_{\ell, h} \right\rangle } }
     & \le \sqrt{2d} \E[\ell]{ \bar{\sigma}_{\ell, h}^2 \cdot 1 \wedge (X_{\ell, h}/\bar{\sigma}_{\ell, h} )^\top {\rmGamma_{\ell, h}} (X_{\ell, h}/\bar{\sigma}_{\ell, h} ) }^{1/2}
  \end{align}
  Recall we choose $\bar{\sigma}^2_{\ell, h} := \E[\ell, h]{\sigma_{\ell, h}^2} \vee \sigma^2_{\min}$ for analysis and thus \cref{eq:one-episode-upper-bound} and \cref{eq:one-episode-upper-bound-abs} are valid.
  Finally, by plugging \eqref{eq:one-episode-upper-bound} into \eqref{eq:decompose-2},
  \begin{align*}
    \E{ \sum_{\ell=1}^{L} \barV[\hat{M}_\ell]{\pi_{\ell}} - \barV[M]{\pi_{\ell}} }
     & = \E{ \sum_{\ell=1}^{L} \E[\ell]{ \barV[\hat{M}_\ell]{\pi_{\ell}} - \barV[M]{\pi_{\ell}} } }                                                                                                                                                         \\
     & \le \sqrt{d} \E{ \sum_{\ell=1}^{L} \sum_{h = 0 }^{H-1} \E[\ell]{ \bar{\sigma}_{\ell, h}^2 \cdot 1 \wedge (X_{\ell, h}/\bar{\sigma}_{\ell, h} )^\top {\rmGamma_{\ell, h}} (X_{\ell, h}/\bar{\sigma}_{\ell, h} ) }^{1/2} }                             \\
     & \le \sqrt{d} \sum_{\ell=1}^{L} \sum_{h = 0 }^{H-1} \E{ \bar{\sigma}_{\ell, h}^2 \cdot 1 \wedge (X_{\ell, h}/\bar{\sigma}_{\ell, h} )^\top {\rmGamma_{\ell, h}} (X_{\ell, h}/\bar{\sigma}_{\ell, h} ) }^{1/2}                                         \\
     & \stackrel{(\star)}{\le} \sqrt{d} \sum_{\ell=1}^{L} \sum_{h = 0 }^{H-1} \E{ \bar{\sigma}_{\ell, h}^2 }^{1/2} \cdot \E{ 1 \wedge (X_{\ell, h}/\bar{\sigma}_{\ell, h} )^\top {\rmGamma_{\ell, h}} (X_{\ell, h}/\bar{\sigma}_{\ell, h} ) }^{1/2}         \\
     & \le \sqrt{d \E{ \sum_{\ell=1}^{L} \sum_{h = 0 }^{H-1} \bar{\sigma}_{\ell, h}^2 } \cdot \E{ \sum_{\ell=1}^{L} \sum_{h = 0 }^{H-1}  1 \wedge (X_{\ell, h}/\bar{\sigma}_{\ell, h} )^\top {\rmGamma_{\ell, h}} (X_{\ell, h}/\bar{\sigma}_{\ell, h} ) } }
  \end{align*}
  Here, $\sigma_{\min} = H$.

\end{proof}

\section{Proof sketch of \cref{thm:inhomo-bayesian-regret-0}}
\label{sec:sketch}
\subsection{Quantify information gain of the environment via posterior variance reduction}

The first step is to mathematically quantify the reduction of uncertainty of the unknown environment when new information gathered in.
We use the quantity \emph{posterior variance} of the unknown model parameters $\rmGamma_{\ell, h} := \var{\theta^*_h \given \hist{\ell}}$ defined in \cref{def:value-cor-feat} as a uncertainty measurement.
From the perspective of value-targeted model learning stated in \cref{sec:value_pers}, with the input as $X_{\ell, h} := \phi_{\hat{V}_{\ell, h+1}}(x_{\ell, h}) = \sum_{s' \in \mathcal{S}} \phi\left(s^{\prime} \mid x_{\ell, h} \right) \hat{V}_{\ell, h+1}\left(s' \right)$ and response as $Y_{\ell, h+1} := \hat{V}_{\ell, h+1}(s_{\ell, h+1})$, we establish \cref{thm:vd} on posterior variance reduction on the unknown model parameters as following:
\begin{align*}
    \E{ \rmGamma_{\ell+1, h} \given \hist{\ell, h} }^{-1} \succeq \rmGamma_{\ell, h}^{-1} + \frac{1}{\E{{\sigma}^2_{\ell, h} \given \hist{\ell, h}}}X_{\ell, h} X_{\ell, h}^\top,
\end{align*}
where $\sigma^2_{\ell, h} = \var{\hat{V}_{\ell, h+1}(s_{\ell, h+1}) \given M, \hist{\ell, h}}$ is the variance of value due to the transitional noise for sampling $s_{\ell, h+1} \sim \trans(x_{\ell, h})$ given the true MDP $M$ and history $\hist{\ell, h} = (\hist{\ell}, \pi_\ell, x_{\ell, 0}, \hat{V}_{\ell, 1}, \ldots, x_{\ell, h}, \hat{V}_{\ell, h+1})$.
This is the key enabling technique with which we can measure the reduction involving with the information of value variance under posterior probability measure without using any concentration bounds.

\subsection{Regret decomposition to cumulative posterior value variance and cumulative potential}

The second step is to decompose the regret into the the so-called cumulative variance and cumulative potentials which will be explained later.
\begin{itemize}
    \item To facilitate the analysis, we define a bounded-below sequence $\bar{\sigma}_{\ell, h}$ such that
          $$\bar{\sigma}^2_{\ell, h}:= \E{ \sigma^2_{\ell, h} \given \hist{\ell, h} } \vee \sigma^2_{\min}. $$
          In the analysis, we have twp different choice of $\sigma_{\min}$, leading to two different bounds.
          \begin{itemize}
              \item Choice 1: choose $\sigma_{\min} := H$
          \end{itemize}
    \item Since the regret is equal to the estimation error in expectation. By \cref{prop:regret-decompose}, we bound the estimation error by
          \begin{align*}
              \E{ \underbrace{\sum_{\ell=1}^L \barV[\hat{M}_{\ell}]{\pi_{\ell}} - \barV[M]{\pi_{\ell}} }_{\textbf{Estimation error}} } \le \sqrt{d \underbrace{ \E{ \sum_{\ell = 1}^{L} \sum_{h=0}^{H-1} \bar{\sigma}_{\ell, h}^2 } }_{\textbf{Cumulative variance}} \underbrace{ \E{ \sum_{\ell = 1}^{L} \sum_{h=0}^{H-1} 1 \wedge \norm{ \frac{X_{\ell, h}}{\bar{\sigma}_{\ell, h}}}^2_{\rmGamma_{\ell, h} } } }_{\textbf{Cumulative potential}} }
          \end{align*}
    \item The cumulative variance term is approximately the expected summation of the value variance  $\var{\hat{V}_{\ell, h+1}(s_{\ell, h+1}) \given M, \hist{\ell, h}}$ due to the transitional noise.
    \item The cumulative potential is the expected summation of the estimation error of $\theta^*_h$ in the direction of the variance-normalized feature $X_{\ell, h}/\bar{\sigma}_{\min}$ conditioned on the historical data available in the beginning of episode $\ell$, i.e. $\norm{ \frac{X_{\ell, h}}{\bar{\sigma}_{\ell, h}}}^2_{\rmGamma_{\ell, h} }$.
\end{itemize}

\subsection{Bounding cumulative variance}

\begin{proposition}[Cumulative variance of value under virtual model]
    \label{prop:sum-virtual-value-variance}
    Recall the definition $\bar{\sigma}_{\ell, h}^2 = \E[\ell, h]{ \sigma_{\ell, h}^2 } \vee \sigma_{\min}^2$
    where $\sigma_{\ell, h}^2 = \var{ \hat{V}_{\ell, h+1}(s_{\ell, h+1}) \given M, \hist{\ell, h} }$.
    \begin{itemize}
        \item[1.] If $\sigma_{\min} = H$, we have
              \begin{align*}
                  \E{ \sum_{\ell = 1}^{L} \sum_{h=0}^{H-1} \bar{\sigma}_{\ell, h}^2 }
                  \le LH^3
              \end{align*}
    \end{itemize}

\end{proposition}

For choice 2 of $\sigma_{\min} = H/\sqrt{d}$: We bound the cumulative variance with $\mathcal{O}(LH^2)$ by \cref{prop:sum-virtual-value-variance} provided the condition $d \gtrsim H$. The proof is to relate the variance of virtual value functions $\var{ \hat{V}_{\ell, h}(s_{\ell, h+1}) \given M, \hist{\ell, h}}$ to the variance of true value functions $\var{\V{M}{\pi_{\ell}, h}(s_{\ell, h+1}) \given M, \hist{\ell, h}}$. The proof sketch is as follows
\begin{itemize}
    \item By \cref{lem:diff-virtual-true}, we upper bound the variance difference
          $$\var{\V{M}{\pi_{\ell}, h}(s_{\ell, h+1}) \given M, \hist{\ell, h}} - \var{ \hat{V}_{\ell, h}(s_{\ell, h+1}) \given M, \hist{\ell, h}},$$ which is a dominated term.
    \item (\textbf{Dominant term.}) Then, by \cref{lem:sum-true-value-variance}, we use a conditioned version of law of total variance to relate the expected summation of variance of value to the variance of summation of rewards:
          \begin{align}
              \label{eq:sketch-ltw}
              \E{ \sum_{h=0}^{H-1} \var{ \V{M}{\pi_\ell, h+1}(s_{\ell, h+1}) \given M, \hist{\ell, h} }  } = \E{ \var{ \sum_{j=0}^{H-1} R( x_{\ell, j} ) \given[\big] M, \hist{\ell, 0} }} \le H^2.
          \end{align}
          A naive upper bound the summation of the value variance will be $H^3$. Therefore, with the time-varying variance information in hand, \eqref{eq:sketch-ltw} has a direct consequence on the $\sqrt{H}$ factor improvement in the final regret bound.
\end{itemize}

\subsection{Bounding cumulative potential}
\begin{proposition}[Cumulative potential]
    \label{prop:cumulative-potential}
    Under the above assumptions and definitions,
    for any $L \in \mathbb{Z}_+$ and any $h \in [H]$,
    \begin{align*}
        \E{\sum_{\ell = 1}^{L} 1\wedge \norm{ \frac{X_{\ell, h}}{\bar{\sigma}_{\ell, h}}}^2_{\rmGamma_{\ell, h} } }
        \le 2 \log \det \left( \mI + \frac{LH^2}{\sigma_{\min}^2} \rmGamma_{1, h} \right)
    \end{align*}
    where $\rmGamma_{1, h}$ is the prior covariance of $\theta^*_h$.
\end{proposition}
If \cref{asmp:bounded-norm} holds, we can further upper bound the RHS by the \cref{fact:bounded}
\begin{align}
    \log \det (\mI + (LH^2/ \sigma^2_{\min}) \rmGamma_{1, h}) \le d \log ( 1 + LH^2B^2/\sigma^2_{\min} )
\end{align}
\begin{itemize}
    \item[1.] $\sigma_{\min} = H$, the upper bound becomes $\mathcal{O} ( d \log L )$
\end{itemize}

\subsubsection{Put everything together}
Putting everything together, by \cref{prop:regret-decompose}
\begin{itemize}
    \item[1.] $\sigma_{\min} = H$,
          we obtained the upper bound $\mathcal{O}(d H^{3/2} \sqrt{T\log T})$ in \cref{thm:inhomo-bayesian-regret-0,rem:prior-free-0}.
\end{itemize}

\section{Technical details for regret decomposition}
\label{sec:est_err_decompose}
\subsection{Generic estimation error decomposition}
\begin{lemma}[Estimation error decomposition]
  \label{lem:decomposition}
  For any MDP $\hat{M}$ and policy $\pi$,
  let $\hat{\trans}$ be the transition model under $\hat{M}$ and $s_1, \ldots, s_{H-1}$ is rolled out under the policy $\pi$ over the MDP $M$.
  Define $\Delta_{h}(s) = \V{\hat{M}}{\pi, h}(s) - \V{M}{\pi, h}(s)$ and we have,
  \begin{align*}
    \Delta_{h}(s_h)
    & = \left(\hat{\trans} - \trans \right) V^{\hat{M}}_{\pi, h+1}( h, s_{h}, \pi(s_h) ) 
    + \Delta_{h+1}(s_{h+1}) + d_{h}(\Delta_{h+1}) \\
    \abs{\Delta_{h}(s_h) }
    & \le \abs{\left(\hat{\trans} - \trans \right) V^{\hat{M}}_{\pi, h+1}( h, s_{h}, \pi(s_h) ) }
    + \abs{ \Delta_{h+1}(s_{h+1}) } + d_{h}(\abs{ \Delta_{h+1}} ).
  \end{align*}
  where $d_{h}(\Delta) = [ \trans \Delta ](h, s_h, \pi(s_h)) - \Delta (s_{h+1})$.
\end{lemma}
\begin{proof}
  Now we derive the decomposition of the value gap,
  \begin{align*}
    \Delta_h(s_h) & = V^{\hat{M}}_{\pi, h}(s_h) -  V^{{M}}_{\pi, h}(s_h) \\
    & = \hat{\trans} V^{\hat{M}}_{\pi, h+1}(h, s_h, \pi(s_h)) - \trans V^{M}_{\pi, h+1}(h, s_h, \pi(s_h)) \\
    & = \left(\hat{\trans} - \trans \right) V^{\hat{M}}_{\pi, h+1}(h, s_h, \pi(s_h)) 
    + \trans \left( V^{\hat{M}}_{\pi, h+1} - V^{M}_{\pi, h+1} \right) (h, s_h, \pi(s_h)) \\
    & = \left(\hat{\trans} - \trans \right) V^{\hat{M}}_{\pi, h+1}(h, s_h, \pi(s_h)) 
    + \trans \Delta_{h+1} (h, s_h, \pi(s_h)) \\
    & =  \left(\hat{\trans} - \trans \right) V^{\hat{M}}_{\pi, h+1}(h, s_h, \pi(s_h)) 
    + \Delta_{h+1} (s_{h+1})  + d_{h}(\Delta_{h+1}).
  \end{align*}
  And similarly the decomposition for absolute value gap,
  \begin{align*}
    \abs{\Delta_h(s_h)} 
    & \le \abs{\left(\hat{\trans} - \trans \right) V^{\hat{M}}_{\pi, h+1}(h, s_h, \pi(s_h)) }
    + \abs{ \trans \left( V^{\hat{M}}_{\pi, h+1} - V^{M}_{\pi, h+1} \right) (h, s_h, \pi(s_h)) } \\
    & \le \abs{\left(\hat{\trans} - \trans \right) V^{\hat{M}}_{\pi, h+1}(h, s_h, \pi(s_h)) }
    +  \trans \abs{\left( V^{\hat{M}}_{\pi, h+1} - V^{M}_{\pi, h+1} \right)} (h, s_h, \pi(s_h)) \\
    & = \abs{\left(\hat{\trans} - \trans \right) V^{\hat{M}}_{\pi, h+1}(h, s_h, \pi(s_h)) }
    + \trans \abs{ \Delta_{h+1} } (h, s_h, \pi(s_h)) \\
    & = \abs{ \left(\hat{\trans} - \trans \right) V^{\hat{M}}_{\pi, h+1}(h, s_h, \pi(s_h)) }
    + \abs{\Delta_{h+1} (s_{h+1}) } + d_{h}(\abs{ \Delta_{h+1} } ).
  \end{align*}
  where $\abs{\Delta}(s) = \abs{ \Delta(s) }$.
\end{proof}

\begin{lemma}
  \label{lem:martingale-diff}
  If a function $\Delta$ is measurable w.r.t the sigma-algebra generated by random variables $Z$ and $\trans(h, \cdot, \cdot) $,
  then, for any $x = (h, s, a) \in \mathcal{X}_h$ and $s'_x \sim \trans(x)$, the conditional expectation
  \begin{align*}
    \E{ \trans \Delta (x) - \Delta (s'_x) \given Z, \trans(h, \cdot, \cdot)} = 0,
  \end{align*}
  where $\trans \Delta (x) := \E[s'_{x} \sim \trans(x)] {\Delta(x)}$.
\end{lemma}

\subsection{Proof of \cref{lem:simulation} in \cref{sec:generic-regret-decompose}}
\label{sec:proof-simulation}

\begin{proof}
  Note $\barV[\hat{M}]{\pi} - \barV[M]{\pi} = \E{ \Delta_0(s_{0}) \given \hat{M}, M, \pi }$. We can verify that $\Delta_h = \V{\hat{M}}{\pi, h} - \V{M}{\pi, h}$ is measurable w.r.t the sigma-algebra generated by $\pi, \trans([h: H), \cdot, \cdot)$ and $\hat{\trans}([h:H), \cdot, \cdot)$.
  Also note that conditioned on $\pi$, the action is $\pi(s_h) = a_h$.
  Therefore, together with \cref{lem:martingale-diff}, the conditional expectation is
  \begin{align*}
    \E{ [\trans \Delta_{h+1}](h, s_h, \pi(s_h)) ) \given \hat{M}, M, \pi} & = \E{ [\trans \Delta_{h+1}](x_h) \given \hat{M}, M, \pi}  = \E{ \Delta_{h+1}(s_{h+1}) \given \hat{M}, M, \pi}.
  \end{align*}
  By recursive application of \cref{lem:decomposition} form $h=0$ to $h= H-1$ and noticing $$\E{d_{h}(\Delta_{h+1}) \given \hat{M}, M, \pi} = 0,$$ we conclude the result.
\end{proof}

\subsection{Proof of \cref{lem:simulation-history} in \cref{sec:value-targeted-model}}
\begin{proof}
  Recall the definition of value-augmented observations $\hat{O}_{\ell, h}$ and the definition of history $\hist{\ell}$ and $\hist{\ell, h}$ in \cref{sec:value_pers}.
  Also recall the short notation $\E[\ell, h]{ \cdot } := \E{\cdot \given \hist{\ell, h}}$ and $\E[\ell]{\cdot} := \E{ \cdot \given \hist{\ell} } $,
\begin{align}
  \E[\ell, h]{ \Delta_{\ell, h}(s_{\ell, h}) } = \E{ \Delta_{\ell, h}(s_{\ell, h}) \given \hist{\ell, h} } = \E{\Delta_{\ell, h}(s_{\ell, h}) \given \hist{\ell}, \pi_{\ell}, \hat{O}_{\ell, h} }
\end{align}
An application of \cref{lem:decomposition} yields,
\begin{align}
  \E[\ell, h]{ \Delta_{\ell, h}(s_{\ell, h}) } 
  & = 
  \E[\ell]{ \left( \hat{\trans}_{\ell} - \trans \right) \hat{V}_{\ell, h+1}(h, s_{\ell, h}, \pi_{\ell}(s_{\ell, h}) ) 
  + \Delta_{\ell, h+1}(s_{\ell, {h+1}}) + d_{\ell, h}(\Delta_{\ell, h+1})\given \pi_{\ell}, \hat{O}_{\ell, h} } \nonumber \\
  & = 
  \E[\ell]{ \left( \hat{\trans}_{\ell} - \trans \right) \hat{V}_{\ell, h+1}(x_{\ell, h} ) 
  + \Delta_{\ell, h+1}(s_{\ell, {h+1}}) + d_{\ell, h}(\Delta_{\ell, h+1})\given \pi_{\ell}, \hat{O}_{\ell, h} } \nonumber \\
  & \stackrel{(a)}{=} 
  \E[\ell]{ \left( \hat{\trans}_{\ell} - \trans \right) \hat{V}_{\ell, h+1}(x_{\ell, h} ) 
  + \Delta_{\ell, h+1}(s_{\ell, {h+1}}) \given \pi_{\ell}, \hat{O}_{\ell, h} } \nonumber \\
  & = 
  \E[\ell, h]{ \left( \hat{\trans}_{\ell} - \trans \right) \hat{V}_{\ell, h+1}(x_{\ell, h} ) 
  + \E[\ell, h+1]{ \Delta_{\ell, h+1}(s_{\ell, {h+1}}) } }
  \label{eq:conditional-recursion}
\end{align}
where the second equality is due to $\pi_{\ell}$-conditioning with short notation $x_{\ell, h} = (h, s_{\ell, h}, a_{\ell, h})$.
The equality $(a)$ holds by \cref{lem:martingale-diff} with the fact that the function $\Delta_{\ell, h+1} = \V{\hat{M}}{\pi_\ell, h+1} - \V{{M}}{\pi_\ell, h+1} = \hat{V}_{\ell, h+1} - \V{{M}}{\pi_\ell, h+1}$ is measurable w.r.t the sigma-algebra generated by $M, \pi_{\ell}, \hat{V}_{\ell, h+1}$ and the fact $M$ contains h-stage transition kernel $\trans(h, \cdot, \cdot)$,
\begin{align*}
  \E[\ell]{ d_{\ell, h}(\Delta_{\ell, h+1}) \given M, \pi_{\ell}, O_{\ell, h} }
  & = \E[\ell]{ \trans \Delta_{\ell, h+1} (x_{\ell, h}) - \Delta_{\ell, h+1}(s_{\ell, h+1}) \given M, \pi_{\ell}, O_{\ell, h-1}, s_{\ell, h}, a_{\ell, h}, \hat{V}_{\ell, h+1} } \\
  & = \E[\ell]{ \trans \Delta_{\ell, h+1} (x_{\ell, h}) - \Delta_{\ell, h+1}(s_{\ell, h+1}) \given M, \pi_{\ell}, \hat{V}_{\ell, h+1}, O_{\ell, h-1}, x_{\ell, h} } \\
  & = 0.
\end{align*}
Further expand the recursion of \cref{eq:conditional-recursion}, we obtain
\begin{align*}
  \E[\ell, h]{ \Delta_{\ell, h}(s_{\ell, h}) } 
  = \E[\ell, h]{ \sum_{j=h}^{H-1} \E[\ell, j]{ \left( \hat{\trans}_{\ell} - \trans \right) \hat{V}_{\ell, j+1}(x_{\ell, j} ) } }.
\end{align*}
Similarly, by applying \cref{lem:decomposition}, we have
\begin{align*}
  \E[\ell, h]{ \abs{ \Delta_{\ell, h}(s_{\ell, h}) } } 
  \le \E[\ell, h]{ \sum_{j=h}^{H-1} \E[\ell, j]{ \abs{ \left( \hat{\trans}_{\ell} - \trans \right) \hat{V}_{\ell, j+1}(x_{\ell, j} ) } } }.
\end{align*}
\end{proof}

\subsection{Proof of \cref{lem:linear-decomposition}: Error decomposition in linear mixture MDPs from value-targeted perspective}

Recall the short notation of estimation error in stage $h$ of episode $\ell$, $\Delta_{\ell, h}(s) = V^{\hat{M}_{\ell}}_{{\pi}_{\ell}, h}(s) - V^{M}_{{\pi}_{\ell}, h}(s)$. We have the following lemmas prepared for \cref{prop:regret-decompose}.

\begin{proof}
For any $x = (h, s, a) \in \mathcal{X}$, recall that $\phi_V(x) = \sum_{s'} \phi(s' \given x) V(s') $.
For any MDP $\hat{M}$ in the class of linear mixture MDPs, the transition kernel can be represented as $\hat{\trans}(x) = \langle \hat{\theta}_h, \phi(\cdot \given x) \rangle$.
Then, we have the following property 
\begin{align*}
  \hat{\trans} V(x) = \sum_{s'\in \states} \left\langle \hat{\theta}_h, \phi(s' \given x)  \right\rangle V(s') = \left\langle \hat{\theta}_h, \sum_{s'\in \states} \phi(s' \given x) V(s') \right\rangle = \left\langle \hat{\theta}_h, \phi_{V}(x) \right\rangle.  
\end{align*}
With the short notation $X_{\ell, h} = \phi_{\hat{V}_{\ell, h+1}}(x_{\ell, h})$, we have $ \left( \hat{\trans}_{\ell} - \trans \right) \hat{V}_{\ell, h+1}(x_{\ell, h} ) = \left\langle \hat{\theta}_{\ell, h} - \theta^*_h, X_{\ell, h} \right\rangle$.
As a corollary of \cref{lem:simulation-history}, 
\begin{align}
  \label{eq:linear-decompose-0}
  \E[\ell, h]{ \Delta_{\ell, h}(s_{\ell, h}) } 
  = \E[\ell, h]{ \sum_{j=h}^{H-1} \E[\ell, j]{ \left( \hat{\trans}_{\ell} - \trans \right) \hat{V}_{\ell, j+1}(x_{\ell, j} ) } } 
  = \E[\ell, h]{ \sum_{j=h}^{H-1} \E[\ell, j]{ \left\langle \hat{\theta}_{\ell, j} - \theta^*_j, X_{\ell, j} \right\rangle } } 
\end{align}
and similarly from \cref{lem:simulation-history},
\begin{align*}
  \E[\ell, h]{ \abs{ \Delta_{\ell, h}(s_{\ell, h}) } }
  \le \E[\ell, h]{ \sum_{j=h}^{H-1} \E[\ell, j]{ \abs{\left\langle \hat{\theta}_{\ell, j} - \theta^*_j, X_{\ell, j} \right\rangle } } } 
\end{align*}

\end{proof}

\subsection{Proof of \cref{cor:linear-decompose}}

\begin{proof}
Since $X_{\ell, h} = \phi_{\hat{V}_{\ell, h+1}}(x_{\ell, h})$ is a deterministic function of $x_{\ell, h}$ and $\hat{V}_{\ell, h+1}$, by recalling the definition of $\hist{\ell, h}$, we conclude $X_{\ell, h}$ is $\sigma(\hist{\ell, h})$-measurable, resulting the following equation,
\begin{align}
  \E[\ell, h]{ \left\langle \theta^*_h, X_{\ell, h} \right\rangle } 
  & = \left\langle \E[\ell, h]{\theta^*_h}, {X_{\ell, h}} \right\rangle
\end{align}
Notice $x_{\ell, h}$ is rolled out from initial stage to the stage $h$ by policy $\pi_{\ell}$ under $M$.
Therefore, conditioned on $\hist{\ell}$, the state-action pair $x_{\ell, h} = (h, s_{\ell, h}, a_{\ell, h})$ is dependent on $ ( \theta^*_0, \ldots, \theta^*_{h-1} )$ and $\hat{\Theta}_{\ell}$ but is independent of $\theta^*_h$ by \cref{asmp:mutual-independence}.
Also notice the fact that conditioned on $\hist{\ell}$, the imagined value function $\hat{V}_{\ell, h+1}$ is dependent on $\hat{\Theta}_{\ell}$ but is independent of $\Theta^*$.
We conclude that $X_{\ell, h} = \phi_{\hat{V}_{\ell, h+1}}(x_{\ell, h})$ is independent of $\theta^*_h$ conditioned on $\hist{\ell}$.

Similarly, conditioned on $\hist{\ell}$, the random variables $x_{\ell, 0}, \hat{V}_{\ell, 1}, \ldots, x_{\ell, h-1}, \hat{V}_{\ell, h}, x_{\ell, h}, \hat{V}_{\ell, h+1}$ and $\pi_{\ell}$ are all independent of $\theta^*_h$.
Thus, we also have
\begin{align*}
    \E[\ell, h]{\theta^*_h} = \E{\theta^*_h \given \hist{\ell}, \pi_{\ell}, \hat{O}_{\ell, h}} = \E{ \theta^*_h \given \hist{\ell} } = \E[\ell]{\theta^*_h}
\end{align*}
By posterior sampling, we have 
$$\E[\ell]{\theta^*_h} = \E[\ell]{\hat{\theta}_{\ell, h}}.$$
It follows that
\begin{align*}
    \E[\ell, h]{ \langle \hat{\theta}_{\ell, h} - \theta^*_h, X_{\ell, h} \rangle }
    = \E[\ell, h]{ \langle \hat{\theta}_{\ell, h} - \E[\ell]{\hat{\theta}_{\ell, h}}, X_{\ell, h} }.
\end{align*}
Then, one can derive from \eqref{eq:linear-decompose-0},
\begin{align*}
  \E[\ell, h]{ \Delta_{\ell, h}(s_{\ell, h}) } 
  = \E[\ell, h]{ \sum_{j=h}^{H-1} \E[\ell, j]{ \left\langle \hat{\theta}_{\ell, j} - \E[\ell]{\hat{\theta}_{\ell, j}}, X_{\ell, j} \right\rangle } } 
\end{align*}
\end{proof}

\section{Proofs of propositions in \cref{sec:sketch}}
\subsection{Proof of \cref{prop:sum-virtual-value-variance}:
Bounding cumulative value variance}
\label{sec:bounding-variance-potential}
\begin{proof}
For the first case $\sigma_{\min} = H$, since $\hat{V}_{\ell, h+1} \le H$ almost surely, we have $\bar{\sigma}_{\ell, h}^2 = H^2$ almost surely, then the cumulative variance is trivially bounded.

For the second case, the Lemma and Corollary used in the proof can be found in \cref{sec:variance-relation}.
  To upper bound the cumulative variance of value under virtual model, we need to first bound the variance of true value by \cref{cor:sum-true-value-variance} and then relate the variance of virtual value to the variance of the true value by \cref{lem:diff-virtual-true} and \cref{prop:regret-decompose}.

  The first step is to upper bound the LHS to the summation of posterior variance of virtual variance. Since $\bar{\sigma}^2_{\ell, h} := \E[\ell, h]{\sigma_{\ell, h}^2} \vee \sigma^2_{\min}$,
  \begin{align}
    \E{ \sum_{\ell = 1}^{L} \sum_{h=0}^{H-1} \bar{\sigma}_{\ell, h}^2 }
    & \le \E{ \sum_{\ell = 1}^{L} \sum_{h=0}^{H-1} \sigma_{\ell, h}^2 + \sigma_{\min}^2 } \nonumber \\
    & = \E{ \sum_{\ell = 1}^{L} \sum_{h=0}^{H-1} \var[\ell,h]{ \hat{V}_{\ell, h+1}(s_{\ell, h+1}) \given M } } + LH \sigma_{\min}^2
    \label{eq:sigma-bar-upper-bound}
  \end{align}
  By \cref{cor:sum-true-value-variance}, there is an sharp upper bound on the cumulative variance of true value,
  \begin{align}
  \label{eq:ltw}
    \E{ \sum_{\ell = 1}^{L} \sum_{h=0}^{H-1} \var[\ell,h]{ \V{M}{\pi_{\ell}, h+1}(s_{\ell, h+1}) \given M } } \le LH^2.
  \end{align}
  Note that an naive upper bound on LHS of \cref{eq:ltw} would be $LH^3$. The sharp $LH^2$ upper is indeed the source of improvement on the dependence of horizon factor $H$ in the final regret bound. 
  Therefore, the second step is to relate the variance of $ \hat{V}_{\ell, h} = \V{\hat{M}^{\ell}}{\pi_{\ell}, h}$ to the variance of $\V{M}{\pi_{\ell}, h}$, i.e.
  \begin{align*}
    \var[\ell,h]{ \V{\hat{M}^{\ell}}{\pi_{\ell}, h+1}(s_{\ell, h+1}) \given M } 
    & =
    \var[\ell,h]{ \V{M}{\pi_{\ell}, h+1}(s_{\ell, h+1}) \given M } + \delta_{\ell, h},
  \end{align*}
  where $\delta_{\ell, h} := \var[\ell,h]{ \V{\hat{M}^{\ell}}{\pi_{\ell}, h+1}(s_{\ell, h+1} \given M } 
  - \var[\ell,h]{ \V{M}{\pi_{\ell}, h+1}(s_{\ell, h+1}) \given M }$ is bounded by \cref{lem:diff-virtual-true},
  \begin{align*}
    \E{ \sum_{\ell = 1}^{L} \sum_{h=0}^{H-1} \delta_{\ell, h} }
    \le 2H \E{ \sum_{\ell=0}^{L-1} \sum_{h=0}^{H-1}  \abs{ \Delta_{\ell, h+1} (s_{\ell, h+1}) } }.
  \end{align*}
  By \cref{prop:regret-decompose},
  \begin{align*}
    \E{ \sum_{\ell=0}^{L-1} \sum_{h=0}^{H-1}  \abs{ \Delta_{\ell, h+1} (s_{\ell, h+1}) } }
    \le H \sqrt{2d \E{ \sum_{\ell = 1}^{L} \sum_{h=0}^{H-1} \bar{\sigma}_{\ell, h}^2 } \E{ \sum_{\ell = 1}^{L} \sum_{h=0}^{H-1} 1 \wedge \norm{ \frac{X_{\ell, h}}{\bar{\sigma}_{\ell, h}}}^2_{\rmGamma_{\ell, h} } } }
  \end{align*}
  Combine the above results and \eqref{eq:sigma-bar-upper-bound}, we have
  \begin{align*}
    \E{ \sum_{\ell = 1}^{L} \sum_{h=0}^{H-1} \bar{\sigma}_{\ell, h}^2 }
    \le LH \sigma_{\min}^2 + LH^2 
    + 2 H^2 \sqrt{2 d \E{ \sum_{\ell = 1}^{L} \sum_{h=0}^{H-1} \bar{\sigma}_{\ell, h}^2 } 
    \underbrace{ \E{ \sum_{\ell = 1}^{L} \sum_{h=0}^{H-1} 1\wedge \norm{ \frac{X_{\ell, h}}{\bar{\sigma}_{\ell, h}}}^2_{\rmGamma_{\ell, h} } } }_{S} }
  \end{align*}
  Since $x \le a \sqrt{x} + b $ implies $x \le \frac{3}{2}( a^2 + b)$, finally we derive
  \begin{align*}
    \E{ \sum_{\ell = 1}^{L} \sum_{h=0}^{H-1} \bar{\sigma}_{\ell, h}^2 }
    \le \frac{3}{2} \left( HL \sigma_{\min}^2 + H^2L + 8 d H^4 S \right),
  \end{align*}
  where $ S = \E{ \sum_{\ell = 1}^{L} \sum_{h=0}^{H-1} 1\wedge \norm{ \frac{X_{\ell, h}}{\bar{\sigma}_{\ell, h}}}^2_{\rmGamma_{\ell, h} } } $.
\end{proof}

\subsection{Proof of \cref{prop:cumulative-potential}: Bounding cumulative potential}
\label{sec:appendix-potential}
The proof of \cref{prop:cumulative-potential} is by the recursive application of \cref{thm:vd} and the following \cref{lem:potential}. Note that various version of the potential lemma \citep{auer2003using,dani2008stochastic,shipra2013thompson,li2019nearly,hamidi2021randomized} is used for analyzing linear bandit and reinforcement learning with function approximation. We apply the following version of potential lemma in the proof of \cref{prop:cumulative-potential}. The proof of \cref{lem:potential} can be found in \cref{sec:potential_lemma}.
\begin{lemma}[Potential lemma~\citep{hamidi2021randomized}]
  \label{lem:potential}
  For $x>0$ and positive semi-definite matrix $\mSigma$, define $f({\mSigma}, x)=\log \det(\mI+x {\mSigma})$. Then, for any vector $V \in \R^{d}$, we have
  \[
  \log \left(1+V^{\top} {\mSigma} V\right)+f\left({\mSigma}^{\prime}, x\right) \leq f\left({\mSigma}, x+V^{\top} V\right)
  \]
  where 
  \[
    \mSigma^{\prime}:=\mSigma-\frac{\mSigma V V^{\top} \mSigma}{1+V^{\top} \mSigma V}=\mSigma^{\frac{1}{2}}\left(\mI-\frac{\mSigma^{\frac{1}{2} } V V^{\top} \mSigma^{\frac{1}{2}}}{1+V^{\top} \mSigma V}\right) \mSigma^{\frac{1}{2}}
  \]
  and equivalently $\mSigma^{\prime-1}=\mSigma^{-1}+V V^{\top},$ using Sherman-Morrison formula.
\end{lemma}
  By the assumption of the feature vector and the fact value is bounded by [0,H], for each episode $\ell$ and each stage
  \begin{align*}
    \norm{ \phi_{V_{M^{\ell}, h+1}^{\pi^{\ell}}}(s^{\ell}_h, a^{\ell}_h) }^2 & = \norm{ \sum_{s' \in \states} \phi( s' \given s^{\ell}_h, a^{\ell}_h ) V_{M^{\ell}, h+1}^{\pi^{\ell}}(s') }^2 \le H^2
  \end{align*}
  Since $\bar{\sigma}^2_{\ell, h} = \max( \sigma^2_{\min}, \E[\ell, h]{\sigma^2_{\ell, h}} ) \ge \sigma^2_{\min} $, we have
  \begin{align*}
    \norm{ \frac{X_{\ell, h} }{\bar{\sigma}_{\ell, h}} } \le \frac{H}{\sigma_{\min}}
  \end{align*}
  For any $x\in (-1, +\infty)$, we have $1 \wedge x \le 2 \log_2(1+x)$ and therefore
  \begin{align*}
    1 \wedge \norm{ \frac{X_{\ell, h} }{\bar{\sigma}_{\ell, h}}}^2_{\rmGamma_{\ell, h} } \le 2 \log \left( 1 + \norm{ \frac{X_{\ell, h} }{\bar{\sigma}_{\ell, h}}}^2_{\rmGamma_{\ell, h} } \right)
  \end{align*}
  Define the short notation
  $\bar{X}_{\ell, h} = X_{\ell, h}/{\bar{\sigma}_{\ell, h}}$, we will prove the proposition by {mathematical induction}:
  For base case $L= 1$,
  \begin{align*}
    \log ( 1 + \norm{\bar{X}_{1, h}}^2_{\rmGamma_{1, h}}) 
    & = \log \det \left( \mI + \rmGamma_{1, h}^{1/2} \bar{X}_{1, h} \bar{X}_{1, h}^\top {\rmGamma_{1, h}^{1/2} } \right) \\
    & \le \log \det \left( \mI +  \frac{H^2}{\sigma^2_{\min}} \rmGamma_{1, h}^{1/2} \mI \rmGamma_{1, h}^{1/2} \right) \\
    &
    = \log \det \left( \mI +  \frac{H^2}{\sigma^2_{\min}} \rmGamma_{1, h} \right)
  \end{align*}
  by the fact $(1 + X^\top X) = \det( \mI + X X^\top )$ and $A A^\top \preceq \norm{A}^2 I$.

  We use the {induction hypothesis for $L-\ell$} with prior $\prob( \cdot \given \hist{\ell + 1})$, that is
  \begin{align*}
    \E[\ell + 1]{ \sum_{k = \ell + 1}^{L} 1 \wedge \norm{ \bar{X}_{k, h}}^2_{\rmGamma_{k, h}} } 
    \le 2 \log \det \left( \mI + \frac{(L - \ell ) H^2}{\sigma_{\min}^2} \rmGamma_{\ell + 1, h} \right)
  \end{align*}
  Define function $f(\mSigma, x) = \log \det(\mI + x\mSigma)$
  and denote $$\rmGamma_{\ell, h}' := \rmGamma_{\ell, h} - \frac{ \rmGamma_{\ell, h} \bar{X}_{\ell, h} \bar{X}_{\ell, h}^\top \rmGamma_{\ell, h} }{ 1 + \bar{X}_{\ell, h}^\top \rmGamma_{\ell, h} \bar{X}_{\ell, h} } .$$
  We have
  \begin{align*}
    \E[\ell, h]{ \sum_{k = \ell + 1}^{L} 1 \wedge \norm{ \bar{X}_{k, h}}^2_{\rmGamma_{k, h}} }  
    &= \E[\ell, h]{ \E[\ell+1]{ \sum_{k = \ell + 1}^{L} 1 \wedge \norm{ \bar{X}_{k, h}}^2_{\rmGamma_{k, h}} } }  \\
    & \le 2 \log \det \left( \mI + \frac{(L-\ell)H^2}{\sigma_{\min}^2} \E[\ell, h]{ \rmGamma_{\ell+1, h} } \right) \\
    & \stackrel{(a)}{\le} 2 \log \det \left( \mI + \frac{(L-\ell)H^2}{\sigma_{\min}^2} \left( \rmGamma_{\ell, h} - \frac{ \rmGamma_{\ell, h} \bar{X}_{\ell, h} \bar{X}_{\ell, h}^\top \rmGamma_{\ell, h} }{ 1 + \bar{X}_{\ell, h}^\top \rmGamma_{\ell, h} \bar{X}_{\ell, h} } \right) \right) \\
    & = 2 f(\rmGamma_{\ell, h}', (L-\ell) H^2/\sigma_{\min}^2)
  \end{align*}
  where (a) is by the condition $\bar{\sigma}_{\ell, h} \ge \E[\ell, h]{\sigma_{\ell, h}}$ and \cref{thm:vd}.

  Now we use potential lemma (\cref{lem:potential}) to step forward,
  \begin{align*}
    \E[\ell, h]{ \sum_{k = \ell}^{L} 1 \wedge \norm{ \bar{X}_{k, h}}^2_{\rmGamma_{k, h}} }
    & = 1 \wedge \norm{ \bar{X}_{\ell, h}}^2_{\rmGamma_{\ell, h}} + \E[\ell, h]{ \sum_{k = \ell + 1}^{L} 1 \wedge \norm{ \bar{X}_{k, h}}^2_{\rmGamma_{k, h}} } \\
    & \le 2 \left( \log( 1 + \bar{X}_{\ell, h}^\top \rmGamma_{\ell, h} \bar{X}_{\ell, h} ) + f(\rmGamma_{\ell, h}', (L-\ell) H^2/\sigma_{\min}^2) \right) \\
    & \stackrel{(b)}{\le} 2 f(\rmGamma_{\ell, h}, (L-\ell+1) H^2/\sigma_{\min}^2)
  \end{align*}
  where (b) is by \cref{lem:potential}.
  Notice the fact that RHS is $\sigma(\hist{\ell})$-measurable,
  by taking $\E[\ell]{\cdot}$ on both sides of the inequality,
  we conclude that the statement holds for $L - \ell + 1$ with prior $\prob(\cdot \given \hist{\ell})$.

\section{Technical lemmas and facts}
\subsection{Posterior variance reduction for general prior and noise distributions}
\label{sec:posterior-vr-general}
Posterior variance reflects how uncertain we are about a random variable based on existed information.
The progress of learning can be measured by how much variance is reduced when new information comes in.
This section provides a tool to track the variance reduction when prior and noise distribution is \emph{not} Gaussian.
For the purpose of simplicity and generality, we state the lemma in a different set of short notations.

Let $(\Omega, \mathcal{F}, \mathbb{P})$ be a probability space and $\mathcal{F}_{0} \subseteq \mathcal{F}_{1} \subseteq$ $\cdots \subseteq \mathcal{F}$ be an increasing sequence of $\sigma$-algebras that are meant to encode the information available up to time $t$.
Let $\theta^{\star}: \Omega \rightarrow \mathbb{R}^{d}$ be the true parameters vector.
Let $X_{t}$ be a random variable that is possibly dependent on the history and some external source of randomness. 
\begin{assumption}
  \label{asmp:Xt-measurable}
  Assume $X_t$ is $\mathcal{F}_{t}$-measurable.
\end{assumption}
More information about $\theta^{\star}$ is then made available sequentially through a sequence of inputs $X_0, X_1, \cdots: \Omega \rightarrow \R^d$ and a sequence of observations $Y_{1}, Y_{2}, \cdots: \Omega \rightarrow \mathbb{R}$.
\begin{assumption}
  \label{asmp:noise-linear-model}
  Assume $Y_{t+1}$ is $\mathcal{F}_{t+1}$-measurable.
  Assume that for all $t \geq 0$,
  \[
  \E[t]{Y_{t+1} \mid \theta^{\star}}
  =\left\langle \theta^{\star}, X_{t} \right\rangle
  \]
  and there exists some random variable $\sigma_t: \Omega \rightarrow \R$,
  \[
    \var[t]{Y_{t+1} \mid  \theta^{\star} } = \sigma_t^{2} \quad a.s.
  \]
\end{assumption}
One possible way to generate the filtration satisfying the above assumptions is as following,
\[
\Fc_t = \sigma \left(X_{0}, Y_{1}, X_1, \ldots, X_{t-1}, Y_{t}, X_t \right).
\]
\begin{blockquote}
  \begin{lemma}[Variance reduction]
    \label{lem:vr-posterior}
    Define $\rmGamma_t := \var[t]{\theta^*}$ as the posterior covariance of $\theta^*$.
    When \cref{asmp:Xt-measurable} and \cref{asmp:noise-linear-model} hold, for all $t \geq 0$
    \begin{align*}
      \E{{\rmGamma}_{t+1} \given \mathcal{F}_{t} } 
      \preceq {\rmGamma}_{t} - \frac{\rmGamma_{t}^{\top} X_{t} X_{t}^{\top} {\rmGamma}_{t}}{ \E{ \sigma_t^{2} \given \Fc_t } + X_{t}^{\top} \rmGamma_{t} X_{t}} \quad a.s.
    \end{align*}
  \end{lemma}
\end{blockquote}

\begin{remark}
    Note that \citep{hamidi2021randomized} give a similar lemma as \cref{lem:vr-posterior}. But they consider a fixed constant upper bound for noise variance instead of the predictable heterogeneous noise variance $\E{ \sigma_t^{2} \given \Fc_t }$, which is essential for the purpose of this work.
\end{remark}

\begin{proof}[Proof of \cref{lem:vr-posterior}]
Denote the short notation for conditional expectation $\E[t]{\cdot } = \E{\cdot \given \Fc_t}$ and conditional variance $\var[t]{\cdot} = \var{\cdot \given \Fc_t }$.
To make things clear, we first prove the claim for $\theta^{\star}$ assuming $\mathbb{E}_{t}\left[\theta^{\star}\right]=0 .$
It suffices to prove that for any $d$-dim real vector $v \in \R^d$,
\begin{align*}
  v^{\top} \mathbb{E}_{t}\left[{\rmGamma}_{t+1} \right] v
  & =\mathbb{E}_{t}\left[\operatorname{Var}_{t+1}\left(\left\langle\theta^{\star}, v\right\rangle\right) \right] \\
  & \leq v^{\top} {\rmGamma}_{t} v
  - \frac{\left(X_{t}^{\top} {\rmGamma}_{t} v \right)^{2}}{ \E[t]{\sigma_t^{2}} + X_{t}^{\top} {\rmGamma}_{t} X_{t}}
\end{align*}
Denoting by $\mathcal{F}_{t}^{A}$ the set of $\mathcal{F}_{t}$-measurable random variables. For any fixed $v \in \mathbb{R}^{d}$,
\begin{align*} 
  \mathbb{E}_{t}\left[\operatorname{Var}_{t+1 }\left(\left\langle\theta^{\star}, v\right\rangle \right) \right] 
  &=\mathbb{E}_{t}\left[\inf _{\xi \in \mathcal{F}_{t+1}^{A}} \mathbb{E}_{t+1} \left[\left(\left\langle\theta^{\star}, v \right\rangle - \xi \right)^{2} \right] \right] \\ 
  & \leq \mathbb{E}_{t}\left[\inf _{a \in \mathbb{R}} \mathbb{E}_{t+1}\left[\left(\left\langle\theta^{\star}, v \right\rangle-a Y_{t+1}\right)^{2} \right] \right] \\ 
  & \leq \inf _{a \in \mathbb{R}} \mathbb{E}_{t}\left[\mathbb{E}_{t+1}\left[\left(\left\langle\theta^{\star}, v \right\rangle-a Y_{t+1}\right)^{2} \right] \right] \\
  & = \inf _{a \in \mathbb{R}} \mathbb{E}_{t}\left[\left(\left\langle\theta^{\star}, v \right\rangle-a Y_{t+1}\right)^{2} \right] 
\end{align*}
By completing the squares, we have the following important equality
\[
  \inf _{a \in \mathbb{R}}\left(\mathbb{E}_{t}\left[\left(\left\langle\theta^{\star}, v \right\rangle-a Y_{t+1}\right)^{2} \right]\right) =
  { \underbrace{\mathbb{E}_{t}\left[\left\langle\theta^{\star}, v \right\rangle^{2} \right]}_{I} -
  \frac{ \overbrace{ \mathbb{E}_{t}\left[\left\langle\theta^{\star}, v \right\rangle Y_{t+1} \right]^{2} }^{II} }
  {\underbrace{\mathbb{E}_{t}\left[Y_{t+1}^{2} \right]}_{III} }}
\]
For the first term $I$, we have
\[
  \mathbb{E}_{t}\left[\left\langle\theta^{\star}, v \right\rangle^{2} \right]
  =
  \mathbb{E}_{t}\left[v^{\top} \theta^{\star} \theta^{\star \top} v\right]
  = 
  v^{\top} \rmGamma_{t} v
\]
For the second expectation in term $II$, the numerator can also be computed as
\begin{align*}
  \mathbb{E}_{t}\left[\left\langle\theta^{\star}, v \right\rangle Y_{t+1} \right]
  &=\mathbb{E}_{t}\left[\mathbb{E}_{t}\left[\left\langle\theta^{\star}, v \right\rangle Y_{t+1} \given \theta^{\star}\right] \right] \\
  &=\mathbb{E}_{t}\left[\left\langle\theta^{\star}, v \right\rangle \cdot \mathbb{E}_{t}\left[Y_{t+1} \given \theta^{\star}\right] \right] \\
  &=\mathbb{E}_{t}\left[\left\langle\theta^{\star}, v \right\rangle\left\langle\theta^{\star}, X_{t}\right\rangle \right] \\
  &=\mathbb{E}_{t}\left[X_{t}^{\top} \theta^{\star} \theta^{\star \top} v  \right] \\
  & \stackrel{(a)}{=} X_{t}^{\top} \mathbb{E}_{t}\left[ \theta^{\star} \theta^{\star \top} \right] v \\
  & = X_{t}^{\top} {\rmGamma}_{t} v,
\end{align*}
where the equality $(a)$ is by \cref{asmp:Xt-measurable} that $X_t$ is $\Fc_{t}$-measurable.
Finally, for the denominator $III$,
\begin{align*}
  \mathbb{E}_{t}\left[Y_{t+1}^{2} \right] 
  & = \mathbb{E}_{t}\left[\operatorname{Var}_{t}\left(Y_{t+1} \given \theta^{\star}\right)
  + \mathbb{E}_{t}\left[Y_{t+1} \given \theta^{\star}\right]^{2} \right] \\
  & = \E[t]{\sigma_t^{2}} + \mathbb{E}_{t}\left[\left\langle\theta^{\star}, X_{t}\right\rangle^{2}  \right] \\
  & = \E[t]{\sigma_t^{2}} + X_{t}^{\top} \rmGamma_{t} X_{t}
\end{align*}

For the case that $\mathbb{E}_{t}\left[\theta^{\star}\right] \neq 0$, we could recenter $\theta^*$ by its posterior mean.
Define $\mu^{\star}:=\theta^{\star}-\mathbb{E}_{t}\left[\theta^{\star} \right]$ and $Z_{t+1}:= Y_{t+1} - \E[t]{Y_{t+1}} = Y_{t+1}-\left\langle \E[t]{\theta^*}, X_{t}\right\rangle .$

Note that $\operatorname{Var}_{t}\left(\mu^{\star} \right) = \operatorname{Var}_{t} \left(\theta^{\star} \right) = \rmGamma_{t}$
and
\begin{align*}
  \mathbb{E}_t \left[Z_{t+1} \given \mu^{\star}\right] 
  & = \mathbb{E}_t \left[Z_{t+1} \given \theta^{\star}\right] 
  = \langle \theta^*, X_t \rangle - \left\langle \E[t]{\theta^*}, X_{t}\right\rangle = \left\langle\mu^{\star}, X_{t}\right\rangle ,\\
  \operatorname{Var}_t \left(Z_{t+1} \given \mu^{\star}\right)
  & = \operatorname{Var}_t \left(Z_{t+1} \given \theta^{\star}\right) 
  = \operatorname{Var}_t \left(Y_{t+1} - \E[t]{Y_t} \given \theta^{\star}\right) 
  = \operatorname{Var}_t \left(Y_{t+1} \given \theta^{\star}\right) .
\end{align*}   
We can apply the result (for the case $\mathbb{E}_{t}\left[\theta^{\star}\right]=0$ ) by replacing $\theta^*$ with $\mu^*$ and replacing $Y_{t+1}$ with $Z_{t+1}$ and then obtain directly
\[
  \mathbb{E}_{t}\left[ \rmGamma_{t+1} \right] = \mathbb{E}_{t}\left[\operatorname{Var}_{t+1}\left(\mu^{\star}\right) \right] \preceq \rmGamma_{t}-\frac{\rmGamma_{t}^{\top} X_{t} X_{t}^{\top} \rmGamma_{t}}{ \E[t]{\sigma_t^{2}} + X_{t}^{\top} \rmGamma_{t} X_{t}}
\]

\end{proof}

\begin{remark}
\label{rem:posterior-vr}
Let $\bar{\sigma}_t$ be any random variable s.t. $\bar{\sigma}^2_t \ge \E{\sigma^2_t \given \Fc_t}$ almost surely.
By sherman-morrison formula, $$\E{\rmGamma_{t+1} \given \Fc_{t} }^{-1} \succeq \rmGamma_t^{-1} + \frac{1}{\bar{\sigma}_t^2} X_t X_t^{\top}.$$
\end{remark}
\begin{remark}
\label{rem:gaussian}
Recall that in the Gaussian prior and Gaussian noise case, $\rmGamma_{t+1}^{-1} = \rmGamma_{t}^{-1} + \frac{1}{\sigma_t^2} X_t X_t^{\top}$.
\end{remark}
\begin{remark}
  \cref{lem:vr-posterior} demonstrates that the posterior covariance reduces in expectation. However, in general, the reduction does not necessarily hold almost surely. 
\end{remark}

\subsection{Proof of \cref{lem:linear-decoupling}}
\label{sec:tech-proof-linear-decoupling}
\begin{proof}
  Let $\left\{e_{i}\right\}_{i=1}^{d}$ be a set of eigenvectors of $\E{ f_{1}\left(X_{1}\right) f_{1}\left(X_{1}\right)^{\top}} $ and an orthonormal basis for $\R^d$, i.e. $\mI_{d \times d} =\sum_{i=1}^{d} e_{i} e_{i}^{\top}$. Then we have,
  \begin{align}
    \E{ \left(f_{1}\left(X_{1}\right)^{\top} f_{2}\left(X_{2}\right)\right)^{2} }
    & = \E{ \left(\sum_{n=1}^{d}\left(f_{1}\left(X_{1}\right)^{\top} e_{n}\right)\left(f_{2}\left(X_{2}\right)^{\top} e_{n}\right)\right)^{2} } \nonumber \\
    & = \sum_{m, n} \E{ \left(f_{1}\left(X_{1}\right)^{\top} e_{m}\right)\left(f_{2}\left(X_{2}\right)^{\top} e_{m}\right) \left(f_{1}\left(X_{1}\right)^{\top} e_{n}\right)\left(f_{2}\left(X_{2}\right)^{\top} e_{n}\right) } \nonumber \\
    & = \sum_{m, n}
    \E{ \left(f_{1}\left(X_{1}\right)^{\top} e_{m}\right)\left(f_{1}\left(X_{1}\right)^{\top} e_{n}\right) }
    \E{ \left(f_{2}\left(X_{2}\right)^{\top} e_{m}\right)\left(f_{2}\left(X_{2}\right)^{\top} e_{n}\right) } 
    \label{eq:prop2}
  \end{align}

    Since $\left\{e_{i}\right\}_{i=1}^{d}$ are orthogonal eigenvectors of $\E{ f_{1}\left(X_{1}\right) f_{1}\left(X_{1}\right)^{\top} },$ we have
    \begin{align*}
      \E{ \left(f_{1}\left(X_{1}\right)^{\top} e_{m}\right)\left(f_{1}\left(X_{1}\right)^{\top} e_{n}\right) } 
      = e_{m}^{\top} \E{ f_{1}\left(X_{1}\right) f_{1}\left(X_{1}\right)^{\top} } e_{n} = 0
    \end{align*}
    for any $m \neq n .$
    Then we can simplify \eqref{eq:prop2} to
    \begin{align*}
      \E{ \left(f_{1}\left(X_{1}\right)^{\top} f_{2}\left(X_{2}\right)\right)^{2} }
      =
      \sum_{n=1}^{d} \E{ \left(f_{1}\left(X_{1}\right)^{\top} e_{n}\right)^{2}}
      \E{ \left(f_{2}\left(X_{1}\right)^{\top} e_{n}\right)^{2} }
    \end{align*}

    Consequently, we have
    \begin{align*}
      \left(\frac{1}{d} \E{
        \abs{ f_{1}\left(X_{1}\right)^{\top} f_{2}\left(X_{1}\right) } } \right)^{2}
      & = \left(\frac{1}{d} \E{
        \abs{ \sum_{n=1}^{d} \left(f_{1}\left(X_{1}\right)^{\top} e_{n}\right)\left(f_{2}\left(X_{1}\right)^{\top} e_{n}\right) } } \right)^{2} \\
      & \le \left(\frac{1}{d} \sum_{n=1}^{d} \E{
        \abs{ \left(f_{1}\left(X_{1}\right)^{\top} e_{n}\right)\left(f_{2}\left(X_{1}\right)^{\top} e_{n}\right) } } \right)^{2} \\
      & \stackrel{(i)}{\le} \frac{1}{d} \sum_{n=1}^{d} \E{ \abs{\left(f_{1}\left(X_{1}\right)^{\top} e_{n}\right)\left(f_{2}\left(X_{1}\right)^{\top} e_{n}\right) } 
      }^{2} \\
      & = \frac{1}{d} \sum_{n=1}^{d} \E{
        \abs{f_{1}\left(X_{1}\right)^{\top} e_{n} } \abs{ f_{2}\left(X_{1}\right)^{\top} e_{n} } }^{2} \\
      & \stackrel{(ii)}{\le}\frac{1}{d} \sum_{n=1}^{d} \E{ \left(f_{1}\left(X_{1}\right)^{\top} e_{n}\right)^{2}}
      \E{ \left(f_{2}\left(X_{1}\right)^{\top} e_{n}\right)^{2} } \\
      & =\frac{1}{d} \E{ \left(f_{1}\left(X_{1}\right)^{\top} f_{2}\left(X_{2}\right)\right)^{2} }
    \end{align*}
The above inequality $(i)$ is by Jensen's inequality and above inequality $(ii)$ is by the Cauchy-Bunyakovsky-Schwarz inequality.
The main lemma is now proved.

Observe that the RHS can be further written as,
\begin{align*}
  \E { \left(f_{1}\left(X_{1}\right)^{\top} f_{2}\left(X_{2}\right)\right)^{2} }
  = \Tr{\left( {\E{ f_1(X_1) f_{1}(X_{1})^{\top}}} {\E{f_{2}(X_{2}) f_2(X_2)^\top} } \right)},
\end{align*}
Specifically, let $X_1 = (X, Z)$, $f_1(x, z) = x, f_2(x, z) = z$, and $X_2$ be a i.i.d copy of $X_1$, we have
  \begin{align*}
    {\mathbb{E}\left[ \abs{X^{\top} Z } \right]^{2} } 
    \leq d 
    \Tr\left( { \mathbb{E}\left[X X^{\top}\right] } { \mathbb{E}\left[Z Z^{\top}\right] } \right).
  \end{align*}
\end{proof}
\begin{remark}
  The idea of using orthonomal basis in the proof is adapted from \citep{kalkanli2020improved}. Our statement is stronger than previous literature since $\E{\abs{f_1(X_1) f_2(X_1) }} \ge \E{{f_1(X_1) f_2(X_1) }}$ and we provide the upper bound on $\E{\abs{f_1(X_1) f_2(X_1) }}$.
\end{remark}

\subsection{Variance relationships}
\label{sec:variance-relation}
\begin{lemma}[Cumulative variance of value under true model $V_{M^*}^{\pi^{\ell}}$]
  \label{lem:sum-true-value-variance}
  For any policy $\pi$, let the sequence $x_h = (h, s_{h}, a_{h})$ from $h=0$ to $h= H-1$ be generated by the interaction of policy $\pi$ and MDP $M$. Define the observation history $O_h = (x_0, \ldots, x_{h}) $, we have
  \begin{align*}
    \var{ \sum_{h=0}^{H-1} R( x_{h} ) \given[\bigg] M, \pi, O_0 }
    = \E{\sum_{h=0}^{H-1} \var{ \V{M}{\pi, h+1}(s_{h+1}) \given M, \pi, O_{h} } \given[\bigg] M, \pi, O_0 }
  \end{align*}
\end{lemma}
\begin{proof}
  By the law of total variance, we can decompose $\var{ \sum_{j=h}^{H-1} R( x_{j} ) \given[\big] M, \pi, O_h } $ as
  \begin{align*}
   \underbrace{\E{ \var{\sum_{j=h}^{H-1} R( x_{j} ) \given[\bigg] M, \pi, O_{h+1} } \given[\bigg] M, \pi, O_h } }_{(a)}
   + \underbrace{\var{ \E{\sum_{j=h}^{H-1} R( x_{j} ) \given[\bigg] M, \pi, O_{h+1} } \given[\bigg] M, \pi, O_h }}_{(b)}
  \end{align*}
  Since $x_{h}$ is deterministic conditioned on $O_{h+1}$, we rewrite $(a)$ as
  \begin{align*}
    (a) & = \E{ \var{\sum_{j=h+1}^{H-1} R( x_{j} ) \given[\bigg] M, \pi, O_{h+1} } \given[\bigg] M, \pi, O_{h} }
  \end{align*}
  Also, we can further rewrite (b) by the definition of $\V{M}{\pi_{\ell}, h+1}$
  \begin{align*}
    (b) & = \var{ R(x_{\ell, h}) + \E{\sum_{j=h+1}^{H-1} R(x_{j}) \given[\bigg] M, \pi, O_{h+1} } \given[\bigg] M, \pi, O_h } \\
    & = \var{ \E{\sum_{j=h+1}^{H-1} R(h, s_{h}, a_{h}) \given[\bigg] M, \pi, O_{h+1} } \given[\bigg] M, \pi, O_h } \\
    & = \var{ \E{\sum_{j=h+1}^{H-1} R(h, s_{h}, \pi(s_{h}) ) \given[\bigg] M, \pi, O_{h+1} } \given[\bigg] M, \pi, O_h } \\
    & = \var{ \V{M}{\pi, h+1}(s_{h+1}) \given[\bigg] M, \pi, O_h }
  \end{align*}
  By recursive application of the above result, we obtain
  \begin{align*}
    \var{ \sum_{j=0}^{H-1} R( x_{j} ) \given[\big] M, \pi, O_0 }
    = \E{\sum_{h=0}^{H-1} \var{ \V{M}{\pi, h+1}(s_{h+1}) \given M, \pi, O_{h} } \given[\bigg] M, \pi, O_0 }
  \end{align*}
\end{proof}

\begin{corollary}
  \label{cor:sum-true-value-variance}
  Recall the definition of $\hist{\ell, h}$, we can apply \cref{lem:sum-true-value-variance} with setting $\pi = \pi_{\ell}$ and get,
  \begin{align*}
    \var[\ell, 0]{ \sum_{j=0}^{H-1} R( x_{\ell, j} ) \given[\big] M }
    = \E[\ell, 0]{\sum_{h=0}^{H-1} \var[\ell, h]{ \V{M}{\pi_\ell, h+1}(s_{\ell, h+1}) \given M } \given[\bigg] M }
  \end{align*}
  Beside, we could derive the following upper bound for summation of variance
  \[
    \E{ \sum_{\ell = 1}^{L} \sum_{h=0}^{H-1} \var[\ell, h]{ \V{M}{\pi_\ell, h+1}(s_{\ell, h+1}) \given M }  }
    = \sum_{\ell = 1}^{L} \E{ \var[\ell, 0]{ \sum_{j=0}^{H-1} R( x_{\ell, j} ) \given[\big] M } }
    \le LH^2.
  \]
\end{corollary}

\begin{remark}
  The proof of \cref{lem:sum-true-value-variance} is by the sequential conditional version of the law of total variance, which is similar to the proof in the frequentist literature \citep{munos1999influence, lattimore2012pac, azar2013minimax}.
\end{remark}

\begin{lemma}[Variance difference]
  \label{lem:diff-virtual-true}
  For each stage $h$ of each episode $\ell$,
  \begin{align*}
    \var[\ell,h]{ \V{\hat{M}_{\ell}}{\pi_{\ell}, h+1}(s_{\ell, h+1}) \given M } 
    - \var[\ell,h]{ \V{M}{\pi_{\ell}, h+1}(s_{\ell, h+1}) \given M } 
    \le 2H \E[\ell, h] { \abs{ \Delta_{\ell, h+1}(s_{\ell, h+1}) } \given M }
  \end{align*}
\end{lemma}
\begin{proof}
  First, we rewrite the difference of the variance,
  \begin{align*}
    & \var[\ell,h]{ \V{\hat{M}_{\ell}}{\pi_{\ell}, h+1}(s_{\ell, h+1}) \given M } 
    - \var[\ell,h]{ \V{M}{\pi_{\ell}, h+1}(s_{\ell, h+1}) \given M } \\
    & = 
    \underbrace{ \E[\ell, h]{ \left( \V{\hat{M}_{\ell}}{\pi_{\ell}, h+1}(s_{\ell, h+1}) \right)^2 \given M } 
    - \E[\ell, h]{ \left( \V{M}{\pi_{\ell}, h+1}(s_{\ell, h+1}) \right)^2 \given M } }_{ (a) } \\
    & \quad \underbrace{ - \E[\ell, h]{ \V{\hat{M}_{\ell}}{\pi_{\ell}, h+1}(s_{\ell, h+1}) \given M }^2 
    + \E[\ell, h]{ \V{M}{\pi_{\ell}, h+1}(s_{\ell, h+1}) \given M }^2 }_{ (b) } .
  \end{align*}
  We further bound $(a)$ by,
  \begin{align*}
    (a) 
    & = \E[\ell, h]{ \left( \left[ \V{\hat{M}_{\ell}}{\pi_{\ell}, h+1} - \V{M}{\pi_{\ell},h+1} \right] (s_{\ell, h+1}) \right) \left( \left[ \V{\hat{M}_{\ell}}{\pi_{\ell}, h+1} + \V{M}{\pi_{\ell},h+1}  \right] (s_{\ell, h+1} ) \right) \given M } \\
    & \le 2H \E[\ell, h]{ \left[ \V{\hat{M}_{\ell}}{\pi_{\ell}, h+1} - \V{M}{\pi_{\ell},h+1} \right]_+ (s_{\ell, h+1}) \given M}.
  \end{align*}
  Similarly, bound $(b)$ by,
  \begin{align*}
    (b) 
    & = \E[\ell, h]{ \left[ \V{\hat{M}_{\ell}}{\pi_{\ell}, h+1} - \V{M}{\pi_{\ell},h+1} \right] (s_{\ell, h+1}) \given M } 
    \E[\ell, h]{ \left[ \V{\hat{M}_{\ell}}{\pi_{\ell}, h+1} + \V{M}{\pi_{\ell},h+1}  \right] (s_{\ell, h+1}) \given M } \\
    & \le 2H \E[\ell, h]{ \left[ \V{\hat{M}_{\ell}}{\pi_{\ell}, h+1} - \V{M}{\pi_{\ell},h+1}  \right] (s_{\ell, h+1}) \given M }_+ \\
    & \le 2H \E[\ell, h]{ \left[ \V{\hat{M}_{\ell}}{\pi_{\ell}, h+1} - \V{M}{\pi_{\ell},h+1}  \right]_+ (s_{\ell, h+1}) \given M } .
  \end{align*}
  Putting all together,
  \begin{align*}
    (a) + (b)
    & = 2H \E[\ell, h]{ \left[ \V{\hat{M}_{\ell}}{\pi_{\ell}, h+1} - \V{M}{\pi_{\ell},h+1} \right]_{+} (s_{\ell, h+1}) + \left[ \V{M}{\pi_{\ell},h+1} - \V{\hat{M}_{\ell}}{\pi_{\ell}, h+1} \right]_{+} (s_{\ell, h+1}) \given M } \\
    & = 2H \E[\ell, h]{ \left| \V{\hat{M}_{\ell}}{\pi_{\ell}, h+1}  (s_{\ell, h+1}) - \V{M}{\pi_{\ell},h+1} (s_{\ell, h+1}) \right| \given M } \\
    & = 2H \E[\ell, h] { \abs{ \Delta_{\ell, h+1}(s_{\ell, h+1}) } \given M }.
  \end{align*}

\end{proof}

\subsection{Proof of \cref{lem:potential}}
\label{sec:potential_lemma}
For completeness, we give a proof of potential lemma which is adpated from \citep{hamidi2021randomized}.
\begin{proof}
  First, assume that $\mSigma$ is invertible. In this case, we have $\mSigma^{\prime-1}=\mSigma^{-1}+V V^{\top},$ using Sherman-Morrison formula. 
  Due to the fact $VV^\top \preceq V^\top V \mI$,
  \begin{align*}
    f \left({\mSigma}, x+V^{\top} V\right)
    & =
    \log \det\left( \mSigma^{\frac{1}{2}} \left( \mSigma^{-1} + x \mI + V^\top V \mI \right) \mSigma^{\frac{1}{2}} \right) \\
    & \ge \log \det\left( \mSigma^{\frac{1}{2}} \left( \mSigma^{-1} + x \mI + V V^\top \right) \mSigma^{\frac{1}{2}} \right) \\
    & = \log \det\left( \mSigma^{\frac{1}{2}} \left( {\mSigma'}^{-1} + x \mI \right) \mSigma^{\frac{1}{2}} \right) 
  \end{align*}
  We can further rewrite,
  \begin{align*}
    \log \det\left( \mSigma^{\frac{1}{2}} \left( {\mSigma'}^{-1} + x \mI \right) \mSigma^{\frac{1}{2}} \right)
    & = \log \det({\mSigma})+\log \det\left({\mSigma'}^{-1}+x \mI\right) \\
    & = \log \det\left({\mSigma}^{\prime}\right)
    -
    { \log \det\left(\mI-\frac{{\mSigma}^{\frac{1}{2}} V V^{\top} {\mSigma}^{\frac{1}{2}}}{1+V^{\top} {\mSigma} V}\right) }
    + 
    \log \det\left({\mSigma'}^{-1}+x \mI\right) \\
    & \stackrel{(c)}{=} {\log \det\left({\mSigma}^{\prime}\right)} - {\log \left(1-\frac{V^{\top} {\mSigma} V}{1+V^{\top} {\mSigma} V}\right)} + {\log \det\left({\mSigma'}^{-1}+x \mI\right) } \\
    & = {\log \left(1+V^{\top} {\mSigma} V\right)} + 
    \underbrace{{\log \det\left({\mSigma}^{\prime \frac{1}{2}}\left({\mSigma'}^{-1}+x \mI\right) {\mSigma}^{\prime \frac{1}{2}}\right)}}_{f(\mSigma', x)}
  \end{align*}
  where $(c)$ is from $\det\left(\mI+Z Z^{\top}\right)=1+Z^{\top} Z$.

  We remain to prove the argument for the case of non-invertible matrix $\mSigma$.
  In this case, for $\epsilon>0$, we define $\mSigma_{\epsilon}=\mSigma+\epsilon \mI$ and 
  \[
    \mSigma_{\epsilon}^{\prime}:=\mSigma_{\epsilon}-\frac{\mSigma_{\epsilon} V V^{\top} \mSigma_{\epsilon}}{1+V^{\top} \mSigma_{\epsilon} V}.
  \]
  Clearly, $\mSigma_{\epsilon}$ is invertible. Therefore, we can apply the previous results to $\mSigma_{\epsilon}$ to obtain
  \[
  \log \left(1+V^{\top} {\mSigma}_{\epsilon} V\right)+f\left({\mSigma}_{\epsilon}^{\prime}, x\right) \leq f\left({\mSigma}_{\epsilon}, x+V^{\top} V\right)
  \]
  The claim then follows the continuity of the above expressions with respect to $\epsilon$ on $[0, \infty]$
\end{proof}

\section{Lower bound conjecture on the dependence of interactions}
\label{sec:lower-bound-conjecture}
For the bandit and RL analysis that relies on confidence bounds, one $O(\sqrt{\log T})$ term arises from a union bound over all $T$ interaction periods.
All previous analysis for reinforcement learning\footnote{A concurrent work~\citep{zhang2021feelgood} established the decoupling coefficient approach for analyzing randomized exploration algorithms for contextual bandits and a specific class of reinforcement learning problems with deterministic transitions.} relies on various types of confidence bounds and thus suffers this additional terms.

A recent improved lower bound $\Omega(d \sqrt{T \log T})$ in linear bandit~\citep{li2019nearly} shows the fundamental difference between problems with the finite action spaces and infinite action spaces that is changing over times.
From their evidence, we conjecture that the lower bound for the class of linear mixture MDPs with changing and infinite action spaces should be $\Omega(dH \sqrt{T \log T})$.
Since being a model based algorithm, PSRL can naturally handle the infinite and changing action sets scenario. Therefore, our analysis for PSRL would match the lower bound up to constants if the conjecture holds.

\begin{conjecture}
    Suppose the number of episodes $L \geq \poly(d, H, B)$ for some large $d$ and $H$.
    There exists a prior distribution over $\Theta$ such that the \cref{asmp:mutual-independence} and
    the \cref{asmp:bounded-norm} holds with $B>1$,
    such that, for any algorithm,
    the expected regret is lower bounded as follows:
    $$
        \mathbb{E}_{\Theta} \Re\left(M_{{\Theta}}, L\right) \geq \Omega(d H \sqrt{T \log T})
    $$
    where $T = L H$ and $\mathbb{E}_{\Theta}$ denotes the expectation over the prior distribution over $\Theta$ and the probability distribution generated by the interconnection of the algorithm and the MDP instance.
\end{conjecture}

\end{document}